%% file: main.tex
\pdfoutput=1

\documentclass[11pt]{article}

\usepackage[]{acl}

\usepackage{booktabs} 
\usepackage{tikz} 

\usepackage{mathtools} 

\usepackage{amssymb}

\input{math_commands.tex}

\usepackage{hyperref}
\usepackage{url}
\usepackage{subcaption}
\usepackage{booktabs} 
\usepackage{graphicx}
\usepackage{amsmath}
\usepackage{amsfonts}
\usepackage{multirow}
\usepackage{wrapfig}
\usepackage{comment}

\usepackage{amssymb,amsthm,mathtools, bbm}
\newtheorem{theorem}{Theorem}[section]
\newtheorem{proposition}[theorem]{Proposition}
\newtheorem{lemma}[theorem]{Lemma}

\newtheorem{definition}[theorem]{Definition}
\usepackage[symbol]{footmisc}

\usepackage{fdsymbol}

\usepackage{lipsum}
\usepackage{cuted}

\usepackage{times}
\usepackage{latexsym}

\usepackage{algpseudocode,algorithm,algorithmicx}

\newcommand*\Let[2]{\State #1 $\gets$ #2}
\algrenewcommand\algorithmicrequire{\textbf{Initialize:}}
\usepackage{tikz}
\usetikzlibrary{decorations.pathreplacing,calc}
\newcommand{\tikzmark}[1]{\tikz[overlay,remember picture] \node (#1) {};}

\newcommand*{\AddNote}[4]{%
    \begin{tikzpicture}[overlay, remember picture]
        \draw [decoration={brace,amplitude=0.5em},decorate,ultra thick,red]
            ($(#3)!(#1.north)!($(#3)-(0,1)$)$) --  
            ($(#3)!(#2.south)!($(#3)-(0,1)$)$)
                node [align=center, text width=2.5cm, pos=0.5, anchor=west] {#4};
    \end{tikzpicture}
}%

\usepackage{amsopn}

\usepackage[T1]{fontenc}
\usepackage{soul,color}
\usepackage{cleveref}

\usepackage[utf8]{inputenc}

\usepackage{microtype}

\DeclareMathOperator*{\OT}{T}
\DeclareMathOperator*{\LSE}{LSE}
\DeclareMathOperator*{\defn}{\overset{def.}{=}}

\DeclareMathOperator*{\Cmat}{C}

\usepackage{mathtools}

\newcommand{\declarelogo}[0]{\includegraphics[height=.02\textwidth]{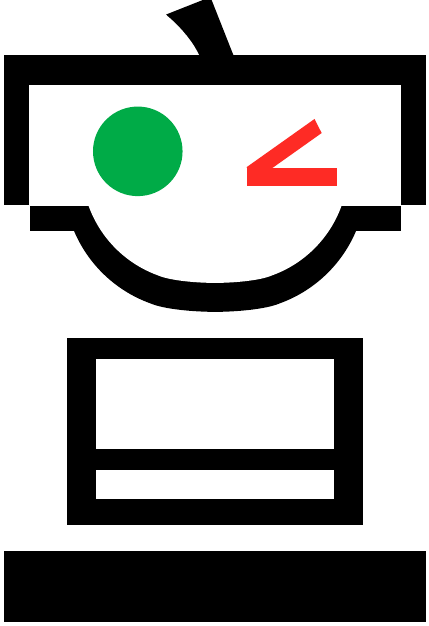}}
\newcommand{\ntulogo}[0]{\includegraphics[height=.02\textwidth]{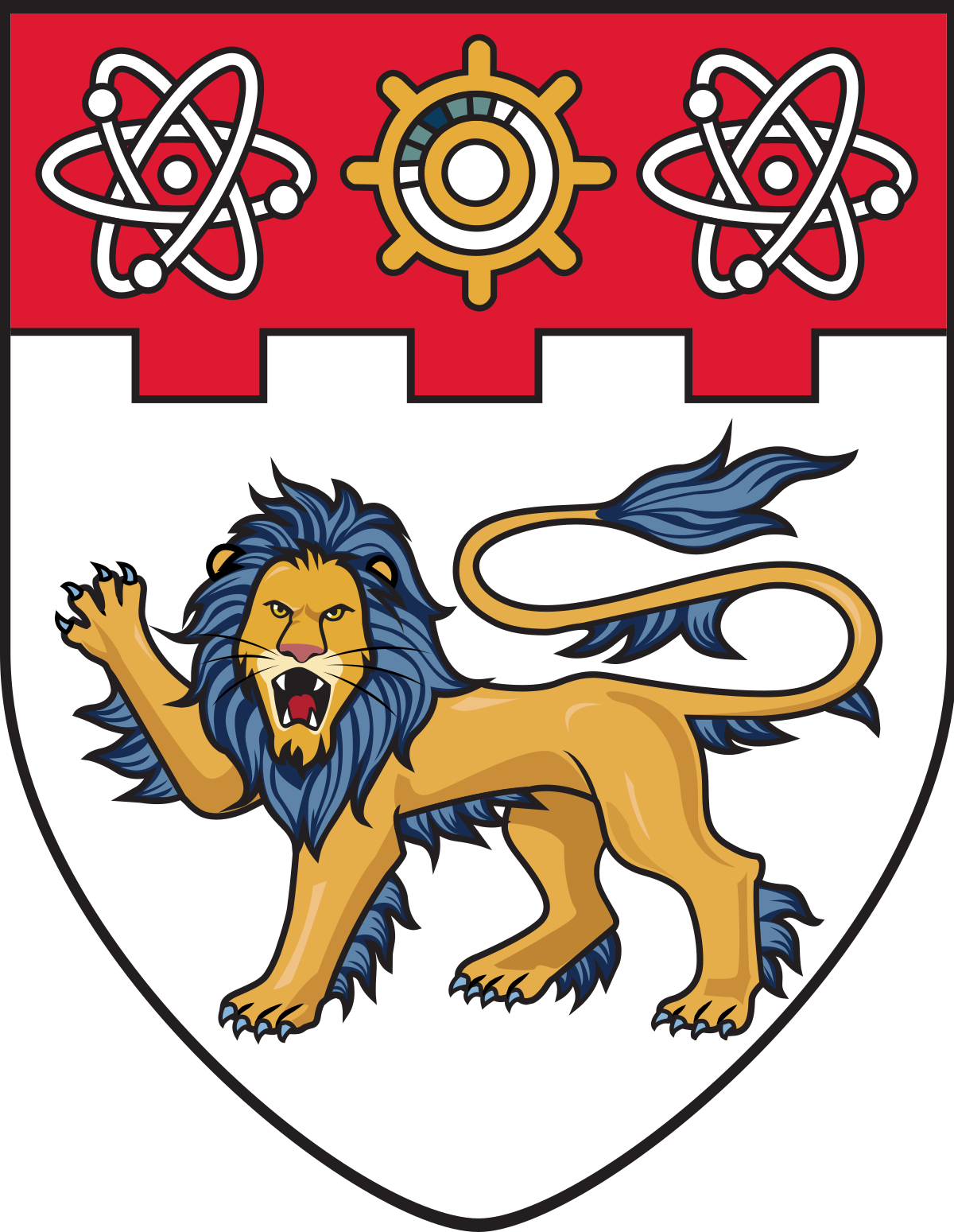}}

\title{KNOT: Knowledge Distillation using Optimal Transport\\ for Solving NLP Tasks}

\author{Rishabh Bhardwaj$^{\declarelogo}$ \hspace{2mm} 
Tushar Vaidya$^{\ntulogo}$ 
\hspace{2mm}Soujanya Poria$^{\declarelogo}$ \\
  $^{\declarelogo}$ Singapore University of Technology and Design, Singapore\\
  $^{\ntulogo}$ Nanyang Technological University, Singapore\\
\texttt{rishabh\_bhardwaj@mymail.sutd.edu.sg, tushar.vaidya@ntu.edu.sg}\\
 \texttt{sporia@sutd.edu.sg}
}

\begin{document}
\maketitle
\begin{abstract}
We propose a new approach, Knowledge Distillation using Optimal Transport (KNOT), to distill the natural language semantic knowledge from multiple teacher networks to a student network. KNOT aims to train a (global) student model by learning to minimize the optimal transport cost of its assigned probability distribution over the labels to the weighted sum of probabilities predicted by the (local) teacher models, under the constraints that the student model does not have access to teacher models' parameters or training data. To evaluate the quality of knowledge transfer, we introduce a new metric, Semantic Distance (SD), that measures semantic closeness between the predicted and ground truth label distributions. The proposed method shows improvements in the global model's SD performance over the baseline across three NLP tasks while performing on par with Entropy-based distillation on standard accuracy and F1 metrics. The implementation pertaining to this work is publicly available at \url{https://github.com/declare-lab/KNOT}.
\end{abstract}

\section{Introduction}
Due to recent technological advancements, more than two-thirds of the world's population use mobile phones\footnote[1]{\url{https://datareportal.com/global-digital-overview}}. A client application on these devices has access to the unprecedented amount of data obtained from user-device interactions, sensors, etc. Learning algorithms can employ this data to provide an enhanced experience to its users. For instance, two users living wide apart may have different tastes in food. A food recommender application installed on an edge device might want to learn from user feedback (reviews) to satisfy the client's needs pertaining to distinct domains. However, directly retrieving this data comes at the cost of losing user privacy \cite{jeong2018communication}.

\begin{figure}
\begin{center}
\includegraphics[width=0.3\textwidth]{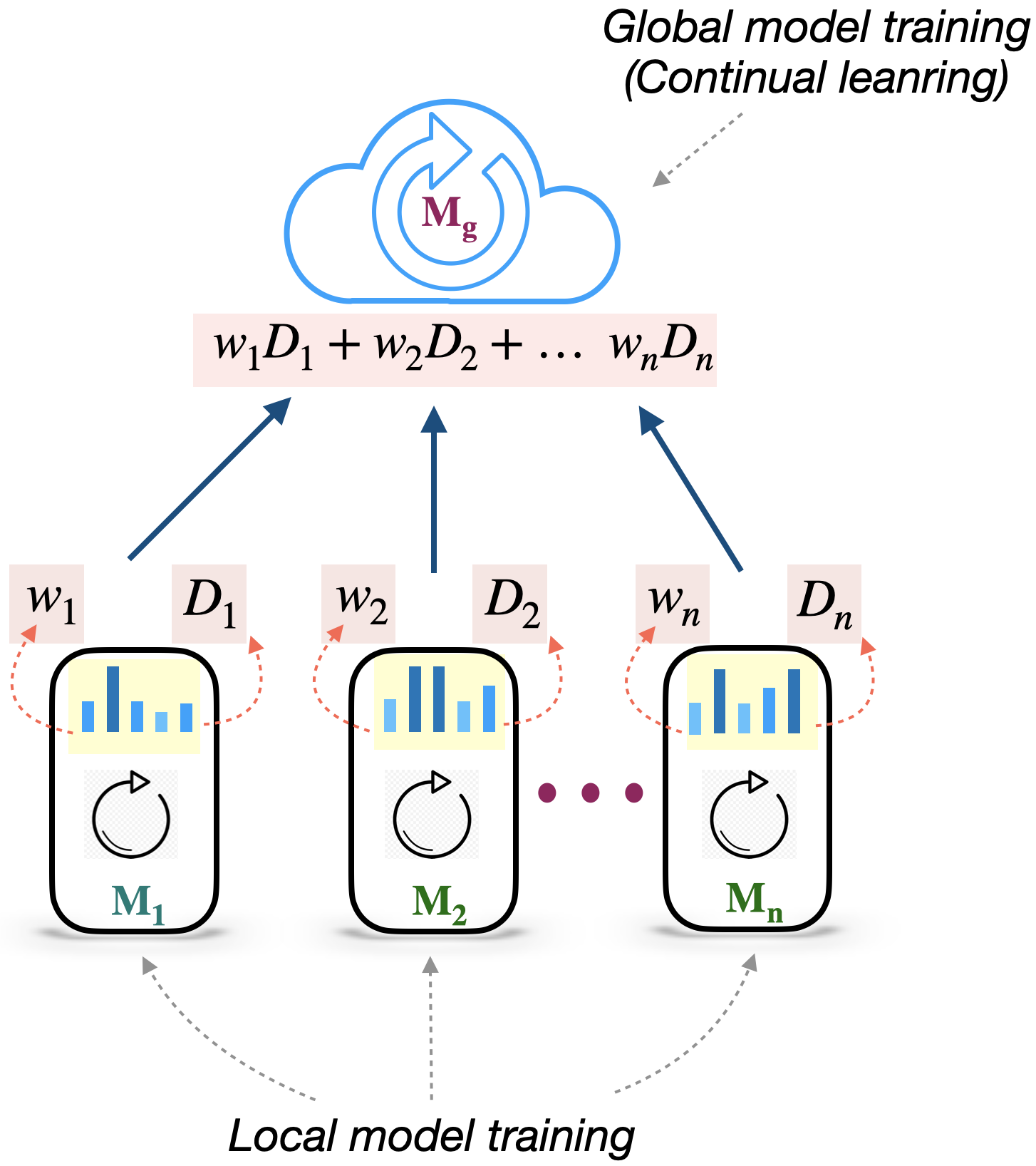}
\end{center}
\caption{KNOT framework: Local models acting as multiple teachers trained locally on user data while the global model acts as a student. The global model can only observe predictions of local models.} \label{fig:modelsetup}
\end{figure}

To minimize the risk of leaking user information, Federated Learning (FL) is a class of algorithms that proposes an alternative learning mechanism \cite{konevcny2016federated, mcmahan2017communication}. The parameters from the teacher networks, i.e., user (domain)-specific \textbf{local models} are retrieved to train a student network, i.e., a user (domain)-generic \textbf{global model}. The classic FL algorithms such as federated averaging and its successors are based on averaging of local model parameters or local gradient updates, and thus only applied when the global and local models possess similar network architectures. Additionally, FL has critical limitations of being costly in terms of communication load with the increase in local model sizes \cite{mohri2019agnostic, li2019fair, jeong2018communication, lin2020ensemble}. Another set of algorithms, Federated Distillation (FD), propose to exchange only the outputs of the local model, i.e, either logits or probability measures whose dimensions are usually much smaller than the parameter size of models themselves \cite{jeong2018communication}. Thus, it enables learning from an ensemble of teacher local models of dissimilar architecture types. In contrast to FL, FD trains the global student model at reduced risk of user privacy, lower communication overhead, and lesser memory space utilization. In this work, we base our problem formulation under FD setting where global model can access the local model outputs without visibility over local user data or model parameters (\Cref{fig:modelsetup}).

Since the Kullback–Leibler (KL) divergence is easy to compute, facilitates smooth backpropagation, and is widely used \cite{murphy2012machine}, it became standard practice to use it to define the objective function in most FD algorithms \cite{gou2021knowledge, lin2020ensemble, jeong2018communication}. However, a critical limitation of such entropy-based losses is that they ignore any metric structure in the label space. For instance, in the task of fine-grained sentiment classification of text, strongly positive sentiment is closer to positive while far from strongly negative sentiment. This information is not fully utilized by the existing distillation algorithms.

Contrary to entropy losses, Optimal Transport (OT) loss admits such inter-class relationships as demonstrated by \citet{frogner2015learning}. Thus, we propose an OT-based approach, KNOT (\textbf{KN}owledge distillation using \textbf{O}ptimal \textbf{T}ransport), for knowledge distillation of natural language semantics encoded in local models under FD setting. To improve the semantic knowledge transfer from local models to the global model acting as multiple teachers-single student framework, we explicitly encode the inter-label relationship in distillation loss in terms of the cost of probability mass transport. Thus, the major contributions of this work are:

\textbf{Contribution:1 (C1)} For the tasks with intrinsic inter-class semantics, we propose a novel optimal transport-based knowledge distillation approach KNOT for distillation from an ensemble of multiple teacher networks under FD setting.

\textit{The problem of bias distillation}: Local models are prone to possess biases that might get transferred to the global model during distillation. The bias in a local model can potentially arise from the user-specific local data that is potentially non-independent and non-identically distributed (non-IID). One such bias is population bias, i.e., the local user may not represent the target overall population \cite{mehrabi2019survey}. This motivates our second contribution:

\textbf{Contribution:2 (C2)} To reduce the potential risk of bias transfer from local models to global, KNOT employs a weighted distillation scheme. The local model with a higher L\textsuperscript{2} distance of predicted probabilities from its intrinsic bias (a distribution) is given higher importance during distillation.

It is important to note that the application of C2 in KNOT does not warrant the C1 setting to be satisfied and vice versa. As shown in \Cref{fig:modelsetup}, the global model $M_g$ aims to learns from the weighted sum of distance $D_1, \ldots, D_n$ with weights being $w_1, \ldots, w_n$. Where $D_k$'s represents the divergence of the global model's prediction from $k\textsuperscript{th}$ local model's prediction. This constitutes our C1. The $w_k$'s are obtained via C2. To validate the fitness of C1 and C2, we also derive generalization bounds of the proposed distillation mechanism.

Next, we introduce a metric called Semantic Distance (SD). It helps us evaluate the semantic closeness of the model's output from the ground truth distribution.

\subsection*{Semantic Distance (SD)}
Most performance metrics, such as accuracy and F1, observe the label with the highest logit (or probability) against ground truth, hence, ignore the overall probability distribution over labels. However, for tasks with inter-class relationships, the predicted distribution shape can be of great importance. Therefore, we define a new performance metric---Semantic Distance (SD)---that measures the semantic closeness of the output distribution against the ground truth. Given a label coordinate space, SD is defined as the mean Euclidean distance of the expected value of output from the ground truth label. For instance, given the sentiment labels \{1, 2, 3, 4, 5\}, the output probabilities of two models m1 and m2 to a strongly negative text input be \{0.2, 0.7, 0.033, 0.033, 0.033\} and \{0.4, 0.1, 0.1, 0.1, 0.3\}, respectively. The argmax of m2 is correct. However, even when the argmax of m1 is incorrect, the expected value of m1, i.e., 1.97 is closer to the ground truth label 1 than m2, i.e., 2.80, and thus semantically more accurate\footnote[2]{Expect value of m1${=} \{1\times0.2+2\times0.7+3\times0.033+4\times0.033+5\times0.033\}$}. A low score denotes a more semantically accurate prediction. The lowest possible value of SD is 0 while the highest possible value depends on the number of labels and their map in the semantic space. For datasets with class imbalance, we first calculate label-wise SD score values and compute their mean to report the SD score on the task.

\begin{figure}
\begin{center}
\includegraphics[width=0.5\textwidth]{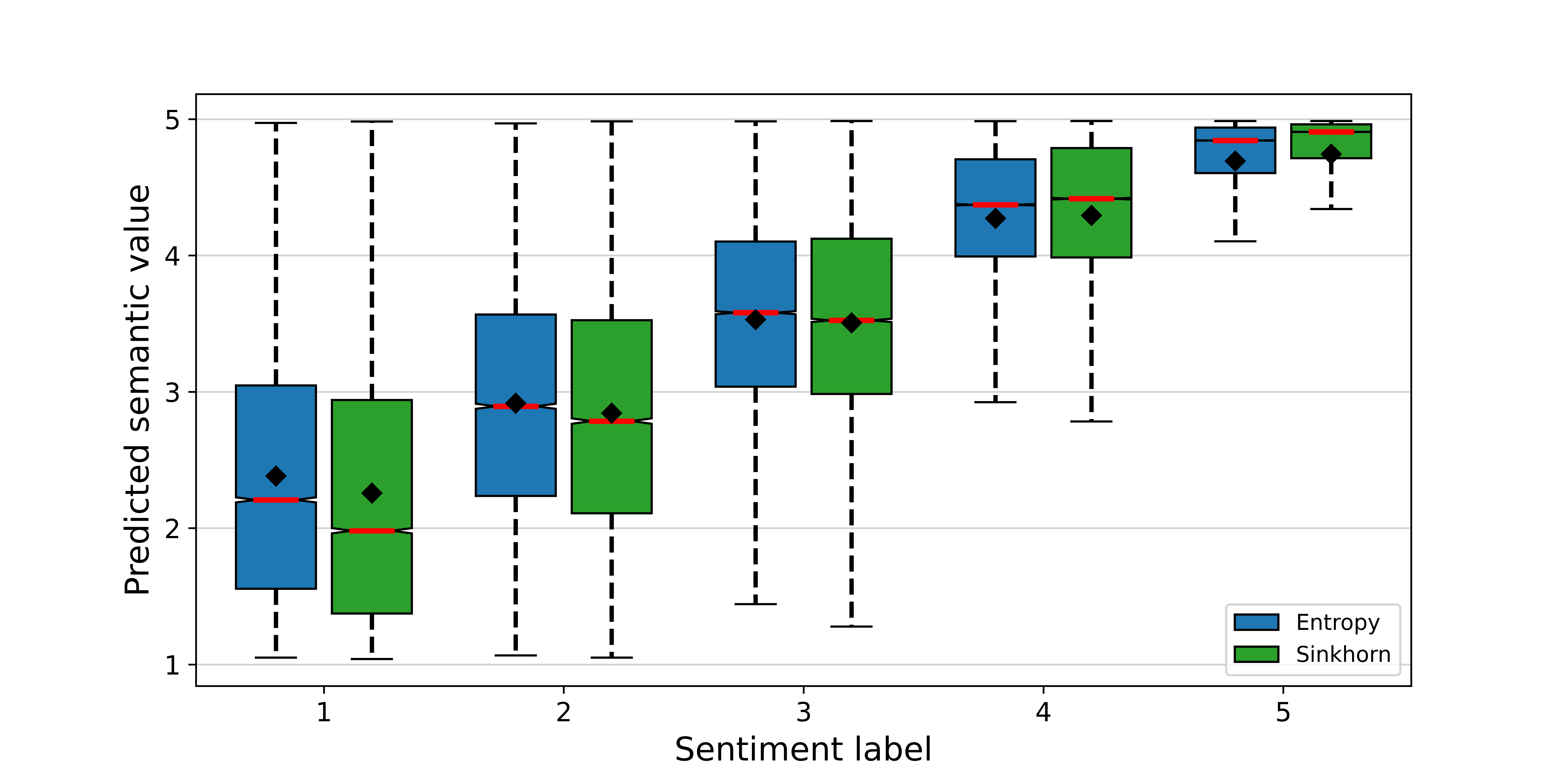}
\end{center}
\caption{Box plot showing expectation of output probabilities for SA task. Horizontal and vertical axes denote ground truth sentiment labels. Sinkhorn denotes loss based on optimal transport.} \label{fig:box_plot}
\end{figure}

\paragraph{Motivation behind SD.} We study the usefulness of the SD metric, for the sentiment analysis, we draw box plots of pretrained global models via Entropy (KL-divergence) and OT-based loss (Sinkhorn). As shown in \Cref{fig:box_plot}, we observe the median SD of OT (green box, red line) is closer to the ground truth sentiment classes---1,2,3, and 5 as compared to the median SD of Entropy (blue box-red line). Similarly, the means (black diamond), the first quartile (25\% of samples), and the third quartile (75\% of samples) for the OT-based model are closer to the ground truth \footnote[3]{The training setup is described later in \Cref{sec:expe_setup}.}.

As a critical finding of this work, OT outperforms baselines on the SD metric. The experiments are carried out on the three natural language understanding (NLU) tasks, i.e., fine-grained sentiment analysis, emotion recognition in conversation, and natural language inference. OT achieves on par with the baselines on the standard performance metrics such as accuracy and Macro F1.

\section{Related work}
There have been many approaches to FL such as local model parameter averaging based on local SGD updates \cite{mcmahan2017communication, lin2018don}. It warrants global and local models to have the same model architecture. Another line of work is multiple-source adaptation formulations where a learner has access to source domain-specific predictor without access to the labeled data. The expected loss is the mixture of source domains \cite{hoffman2018algorithms}. Even though the formulation is close, our solution to the problem is different as we do not have access to the local or global data domain distribution. In Natural Language Processing (NLP), \citet{hilmkil2021scaling, lin2020ensemble, liu2020federated} fine-tune Transformer-based architecture in the federated setting. However, they do not leverage label space semantics and the analysis is restricted to small-scale datasets.

Closest to our work aims to improve local client training based on local data heterogeneity \cite{li2018federated, nedic2020distributed}. Knowledge distillation aims to transfer knowledge from a (large) teacher model to a (smaller) student model \cite{hinton2015distilling, bucil2006model}. Given the output logit/softmaxed values of the teacher model, the student can imitate the teacher's behavior \cite{romero2014fitnets, tian2019contrastive}. A few works are dedicated to the distillation of the ensemble of teacher models to the student model. This includes logit averaging of teacher models \cite{you2017learning, furlanello2018born} or feature level knowledge extraction \cite{park2019feed, liu2019knowledge}.

To the best of our knowledge, there is no prior work that aims to leverage OT to enhance the distillation of semantic knowledge in local models under the FD paradigm. We use standard and widely used entropy-based loss (KL-divergence) as our baseline to compare with C1. We also construct two baselines for confidence score calculation from the prior works, i.e., logit averaging and weighting scheme based on local model dataset size \cite{mcmahan2017communication}. This is to compare the contribution of C2.

\section{Optimal Transport} \label{OT}
Traditional divergences, such as KL, ignore metric structure in the label space $\mathcal{Y}$. In other words, they do not allow the incorporation of inter-class relationships in the loss function. Contrary to this, optimal transport metrics can be extremely useful in defining inter-class semantic relationships in the label space. \footnote[4]{OT offers an additional advantage when measures have non-overlapping support \cite{peyre2019computational}.} The proposed approach, KNOT, utilizes OT in semantic FD settings for tasks, such as sentiment analysis, where inter-class relations can be encoded in the label space. Specific to studied classification problems, we focus on discrete probability distributions. Assume the label space $\mathcal{Y}$ possess a metric $d_{\mathcal{Y}}(\cdot,\cdot)$ that establishes the semantic similarity between labels. The original OT problem is defined as a linear program \cite{bogachev2012monge}. Let $\mu_i$ and $\nu_j$ be the probability masses respectively applied to label $i \in \mathcal{Y}_s$ and label $j \in \mathcal{Y}_t$. Let $\pi_{i,j}$ be the transport assignment from the label $i$ to $j$ that costs $C_{(i,j)}$, i.e., an element of the cost matrix $\Cmat$. We denote Frobenius inner product by $\langle \cdot, \cdot \rangle$. Rather than work with pure OT (Wasserstein) distances, we will restrict our attention to plain regularized OT, i.e, vanilla Sinkhorn distances. The primal goal is to find the plan $\pi \in \Pi(\mu,\nu)$ that minimizes the transport cost

\begin{definition}{Vanilla Sinkhorn Distance}
{\small
\begin{align} \label{eq:plain}
    \begin{split}
    {\OT}_\varepsilon(\mu, \nu) &\defn \min_{\pi \in \Pi(\mu, \nu)} \langle \pi, {\Cmat}\rangle + \varepsilon \KL(\pi,\mu \otimes \nu)\\
    \Pi(\mu, \nu)&=\{\pi \in (\mathbb{R}_{+})^{n_s\times n_t}
    \;\big|\;
    \sum_{j \in \mathcal{Y}_t} \pi_{(i,j)} = \mu_i \\ 
    &\hspace{2.5cm}\textnormal{ and } \sum_{i \in \mathcal{Y}_s} \pi_{(i,j)} = \nu_j\}
    \end{split}
\end{align}
}
\end{definition}
where, 
{\small
\begin{align*}
    \KL(\pi,\mu \otimes \nu)= \sum_{i,j}[\pi_{i,j} \log \frac{\pi_{i,j}}{\mu_i \nu_j} - \pi_{i,j}+ \mu_i\nu_j],
\end{align*}
}

$n_s = |\mathcal{Y}_s|$ and $n_t = |\mathcal{Y}_t|$. For the considered classification tasks, $\mathcal{Y}_s = \mathcal{Y}_t = \mathcal{Y}$.

Entropic regularisation of OT convexifies the loss function and thus is a computational advantage in computing gradients \cite{luise2018differential, peyre2019computational, feydy2019interpolating}. As $\varepsilon \rightarrow 0^+$, we retrieve the unregularized Wasserstein distance.

\section{Methodology} \label{sec:method}

\subsection{Problem framework}
The main participants in KNOT framework are: 1) a set of $K$ local models $\{\mathcal{M}_k\}_{k \in \mathcal{K}}$, $\mathcal{K}=\{1, \ldots, K\}$, and 2) a global model $\mathcal{M}_g$. We denote the set of local models $\{\mathcal{M}_k\}_{k \in \mathcal{K}}$ by $\mathcal{M}_{\mathcal{K}}$.

\begin{itemize}
    \item A local model learns a user-specific hypothesis $h_{k} \in \mathcal{H}_k$ on the $K\textsuperscript{th}$ user-generated data.
    \item The global model $\mathcal{M}_g$ aims to learn a user-generalized hypothesis $h_{\theta} \in \mathcal{H}_g$ that exists on the central application server.
\end{itemize}

\paragraph{Learning goal.} For a given input sample, $\mathcal{M}_g$ has access to the predictions of the local models. Thus, global model training can benefit from the hypotheses of local models collectively denoted by $h_\mathcal{K}$. However, since we aim for secure distillation (FD), $\mathcal{M}_g$ can not retrieve the local models' parameters $\mathcal{M}_{\mathcal{K}}$ or the user-generated data\footnote[5]{The user-generated data is available to user-specific local model only.}. The global model is generally preoccupied with the knowledge generalizable across the users. The distillation task aims to merge the (semantic) knowledge of local models $h_\mathcal{K}$ into the knowledge of global model $h_\theta$. The knowledge transfer happens with the assistance of a transfer set.

\paragraph{Transfer set.} It is the set of unlabeled i.i.d. samples that create a crucial medium for transferring the knowledge from the local models to the global model. To facilitate the knowledge transfer, we obtain the soft labels from local model predictions on the transfer set. The labels are used as ground truth for $\mathcal{M}_g$ and $h_\theta$ is tuned to minimize the discrepancy between global model $\mathcal{M}_g$'s output and the soft labels.

Since only the $h_\mathcal{K}$ is shared with $\mathcal{M}_g$ to tune $h_\theta$, the local models can have heterogeneous architectures. This is useful when certain client devices do not have enough (memory and compute) resources to run large model architectures. It is noteworthy that our approach, KNOT, satisfies this property, however, we do not explicitly perform experiments on heterogeneous local model architectures.

\begin{table*}[]
\begin{center}
\resizebox{1\textwidth}{!}{
\begin{tabular}{|c|c|c|c|c|c|c|c|c|c|c|c|c|c|}
\hline
 & \multicolumn{4}{c|}{SA} & \multicolumn{3}{c|}{ERC} & \multicolumn{6}{c|}{NLI} \\ \hline
 & Cell & Cloths & Toys & Food & IEMOCAP & MELD & DyDa$_{\{1,2,3\}}$ & Fic & Gov & Slate & Tele & Trv & SNLI \\ \hline

train & 133,574 & 19,470 & 116,666 & 397,917 & 3,354 & 9,450 & 21,680 & 77,348 & 77,350 & 77,306 & 83,348 & 77,350 & 549,367 \\ \hline

valid & 19,463 & 28,376 & 17,000 & 57,982 & 342 & 1,047 & 2,013 & 5,902 & 5,888 & 5,893 & 5,899 & 5,904 & 9,842 \\ \hline

test & 37,784 & 55,085 & 33,000 & 112,555 & 901 & 2,492 & 1,919 & 5,903 & 5,889 & 5,894 & 5,899 & 5,904 & 9,824 \\ \hline \hline

score & 0.49 & 0.52 & 0.49 & 0.64 & 0.55 & 0.45 & 0.38, 0.34, 0.40 & 0.63 & 0.65 & 0.62 & 0.65 & 0.63 & 0.85 \\ \hline

\end{tabular}}
\caption{Data statistics and performance of local models: For SA and ERC tasks, the score denotes the Macro F1 performance metric, while it denotes the Accuracy metric for NLI.} \label{tab:stats_F1}
\end{center}
\end{table*}

\begin{figure*}
\begin{center}
\includegraphics[width=0.8\textwidth]{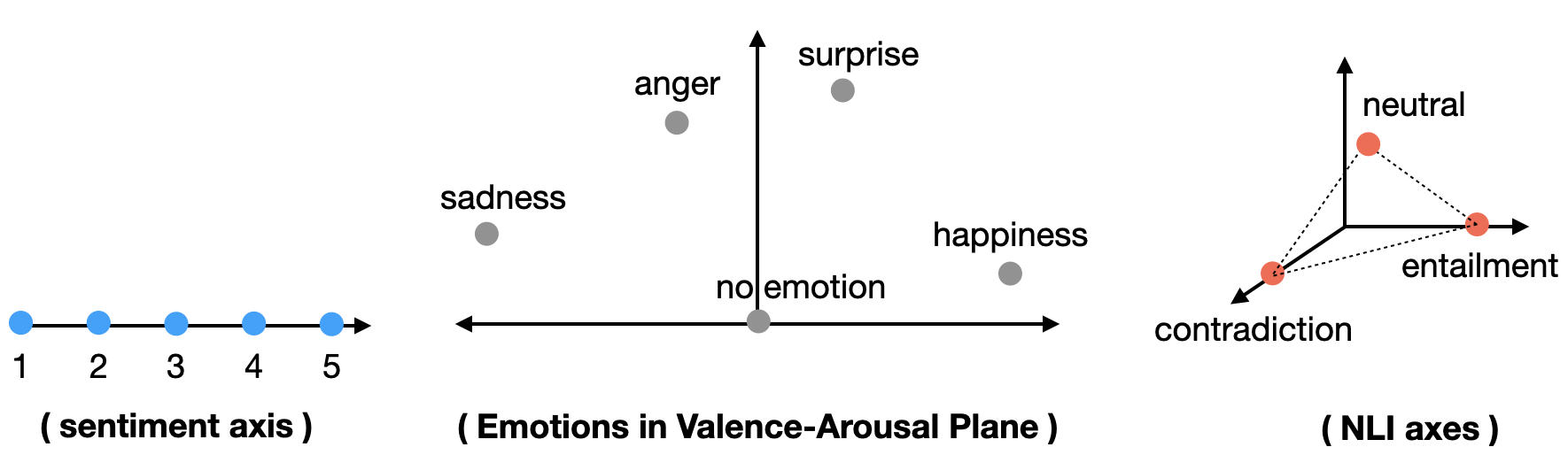}
\end{center}
\caption{Semantic coordinates of SA, ERC, and NLI.} \label{fig:sem_cord}
\end{figure*}

\subsection{Ensemble distillation loss}
For demonstration, we consider a user application that performs sentiment classification task on user-generated text $x^{(i)}=(x_1^{(i)}, x_2^{(i)}, \ldots, x_n^{(i)}) \in \mathcal{X}$ into its sentiment $y^{(i)} \in \mathcal{Y}$, where $\mathcal{X}$ denotes the input space of all possible text strings and the label space defines as

\begin{quote}
$\mathcal{Y}$=\{1(strong negative), 2(weak negative), 3(neutral), 4(weak positive), 5(strong positive)\}.
\end{quote}

In this work, all the hypotheses are of the form $h: \mathcal{X} \mapsto$ {\small $\Delta^\mathcal{Y}$}, {\small$\Delta^\mathcal{Y}$} denotes a probability distribution on the set of labels $\mathcal{Y}$. Under KNOT, we propose a learning algorithm that runs on the central server to fit {\small$\mathcal{M}_g$}'s parameters $\theta$ by receiving predictions such as (softmaxed) logits from $h_{\mathcal{K}}$. Without the loss of generality, the goal is to search for a hypothesis $h_{\hat{\theta}}$ that minimizes the empirical risk
\begin{align}
    h_{\hat{\theta}} = \underset{h_{\theta} \in \mathcal{H}}{\argmin} \Bigg\{\hat{\mathbb{E}}_{S} \big[\;\mathcal{L}_\varepsilon\big(h_{\theta}(x), h_\mathcal{K}(x)\big)\;\big] =\\ \frac{1}{N} \sum_{i=1}^{N} \mathcal{L}_\varepsilon\big(h_{\theta}(x^{(i)}), h_\mathcal{K}(x^{(i)})\big)\Bigg\}. \nonumber
\end{align}

The loss is defined as
\begin{gather} \label{eq:loss_final}
\mathcal{L}_\varepsilon \big( h_{\theta}(x^{(i)}), h_\mathcal{K}(x^{(i)}) \big) = \\ \sum_{k \in \mathcal{K}} \;W_{B_k}\big( h_{k}({x^{(i)}}) \big) \;\;{\OT}_\varepsilon\big(h_\theta({x^{(i)})},\; h_{k}({x^{(i)}})\big) \nonumber
\end{gather}
where ${\OT}_\varepsilon(\cdot,\cdot)$ is the discrepancy  between the two probability measures as its arguments; $W_{B_k}(.)$ is the sample-specific weight assigned to the $k\textsuperscript{th}$ local model's prediction. Next, we elaborate on the the functions ${\OT}_\varepsilon(\cdot,\cdot)$ and $W_{B_k}(.)$ which are crucial for the KNOT algorithm.

\section{Sinkhorn-based distillation}
As an entropy-based loss, we adopt KL divergence. As discussed in \Cref{OT}, we employ Sinkhorn distance to implement OT-based loss which is the proposed KNOT algorithm.

\subsection{Unweighted distillation} \label{UnifDist}
For a text input $x^{(i)}$ from the transfer set, $\OT_\varepsilon\big(h_\theta({x^{(i)})},\; h_{k}({x^{(i)}})\big)$ measures the Sinkhorn distance between the probability output of global model $h_\theta({x^{(i)})}$ and $k\textsuperscript{th}$ local model $h_k({x^{(i)})}$. In \Cref{eq:loss_final}, the sample-wise distance is computed between $h_\theta({x^{(i)})}$ and a probability distribution from the set $h_\mathcal{K}$. A simple approach to fit the global hypothesis $h_\theta$ is to uniformly distill the knowledge from user-specific hypotheses, thus $W_{B_k}\big(h_{k}({x^{(i)}})\big)=1 \;\forall \;k \in \mathcal{K}, \;i \in [N]$.

\subsection{Weighted distillation}\label{bias}
The user-generated local datasets are potentially non-IID with respect to the global distribution and possess a high degree of class imbalance \cite{weiss2001effect}. As each local model $\mathcal{M}_k$ is trained on samples from potentially non-IID and imbalance domains, they are prone to show skewed predictions. The unweighted distillation tends to transfer such biases. One might wonder ``for a given transfer set sample, which local model's prediction is reliable?''.  Although an open problem, we try to answer it by proposing a local model (teacher) weighting scheme. It calculates the confidence score of a model's prediction and performs weighted distillation---weights being in positive correlation with the local model's confidence score. Next, we define the confidence score.

\paragraph{Confidence score (L\textsuperscript{2})}
For a given sample $x$ from the transfer set, the skew in a local model's prediction $h(x)$ can help determine the confidence ($W(\cdot)$ in \Cref{eq:loss_final}) with which it can transfer its knowledge to the global model. However, the local models can show skewed predictions due to training on an imbalanced dataset or chosen capacity of the hypothesis space which can potentially cause the local model to overfit/underfit on user data \cite{caruana2001overfitting}. For instance, a model has learned to misclassify negative sentiment as strongly negative samples owing to a high confusion rate. Such models are prone to show inference time classification errors with highly skewed probabilities. Thus, confidence scoring based on the probability skew may not be admissible. Hence, we incorporate L\textsuperscript{2} confidence for confidence calculation. For a given sample, we define the model's L\textsuperscript{2} confidence score $W_{B}(h(x))$ as the Euclidean distance of its output probability distribution from the probability bias $B$. We define probability bias $B$ of a local model as the prediction when a model $h$ receives random noise at the input. For classification tasks, random texts are generated by sampling random tokens from the vocabulary. Let $h(x) \in \mathcal{Y}$ denotes the predicted distribution of a model for an input text $x$:
\begin{flalign*}
     b_{l \in \mathcal{Y}} &\coloneqq \mathbb{E}_{x \sim \mathcal{N}}[h(x) = l] & \text{{\small ($\mathcal{N}$: the distribution of noise)}}\\
    B &\coloneqq (b_1, \ldots, b_{|\mathcal{Y}|}) &\text{{\small(model probability bias)}}
\end{flalign*}
\begin{definition}
We define $\mathbf{W_{B_k}(h_k)}$ as the L\textsuperscript{2} distance of $k\textsuperscript{th}$ model's prediction from its probability bias $B_k$.
\end{definition}

We provide a detailed analysis of the L\textsuperscript{2}-based confidence metric in the Appendix.

\subsection{Statistical properties}
We derive generalization bounds for the distillation with pure OT loss, i,e, Wasserstein Distance. Let the samples 

$S=\{(x^{(1)},y^{(1)}), \ldots, (x^{(N)},y^{(N)})\})$ 
\\
be IID from the domain distribution of the transfer set and $h_{\hat{\theta}}$ be the empirical risk minimizer. Assume the global hypothesis space $\mathcal{H}_g = \mathfrak{s} \; \circ \;\mathcal{H}_g^{o}$, i.e., composition of softmax and a hypothesis $\mathcal{H}_g^{o}: \mathcal{X} \mapsto \mathbb{R}^{|\mathcal{Y}|}$, that maps input text to a scalar (logit) value for each label. Assuming vanilla Sinkhorn with $\varepsilon \rightarrow 0^+$, we establish the property for 1-Wasserstein.

\begin{theorem} If the global loss function (as in \Cref{eq:loss_final}) uses unregularized 1-Wasserstein metric between predicted and target measure, then for any $\delta > 0$, with probability at least 1-$\delta$ \label{th:risk}
{\small
\begin{flalign*}
    \mathbb{E}\big[\mathcal{L}(h_{\hat{\theta}}(x), h_\mathcal{K}(x))\big] &\leq \underset{h_\theta \in \mathcal{H}_g^o}{\inf} \hat{\mathbb{E}}\big[\mathcal{L}(h_{\theta}(x), h_\mathcal{K}(x))\big] + \\ &32\times|\mathcal{Y}|\times\mathfrak{R}_N(\mathcal{H}_g^{o}) + \\
    &2C_M|\mathcal{Y}| \sqrt{|\mathcal{Y}| \frac{log 1/ \delta}{2N}}
\end{flalign*}
}%
\end{theorem}

where $\mathfrak{R}_N(\mathcal{H}_g^{o})$, decays with $N$, denotes Rademacher complexity \cite{bartlett2002rademacher} of the hypothesis space $\mathcal{H}_g^o$. $C_M$ is the maximum cost of transportation within the label space. In the case of SA, $C_M = 4, |\mathcal{Y}|=5$. The expected loss of the empirical risk minimizer $h_{\hat{\theta}}$ approaches the best achievable loss for $\mathcal{H}_g$. The proof of theorem \Cref{th:risk} and method to compute gradient are relegated to the Appendix.

\begin{table*}[]
\begin{center}
\resizebox{0.95\textwidth}{!}{
\begin{tabular}{c|ccccc|ccccc}
\hline
\multirow{3}{*}{Algorithm} & \multicolumn{5}{c|}{F1 Score}  & \multicolumn{5}{c}{ Semantic Distance}    \\ 
    & \multicolumn{3}{c}{-------Local-------} & Global & \multirow{2}{*}{ALL} & \multicolumn{3}{c}{-------Local-------} & Global & \multirow{2}{*}{ALL} \\ 
    & Cloths   & Toys   & Cell   & Food   &    & Cloths   & Toys   & Cell   & Food   &       \\ 
    Entropy-A  & 0.48 & 0.44 & 0.47 & 0.52 & 0.50  & 0.77 & 0.87 & 0.76 & 0.79 & 0.79   \\
    
    Entropy-D  & 0.48 & 0.44 & 0.46 & 0.56 & 0.52  &0.77 &0.86 &0.77 &0.71 &0.78   \\ 

    Entropy-U  & 0.47 & 0.43 & 0.47 & 0.50 & 0.49  &0.79 &0.90 &0.78 &0.82 &0.83   \\ 

    Entropy-E  & 0.49 & 0.46 & 0.48 & 0.55 & 0.52  & 0.74 & 0.80 & 0.74 & 0.71 & 0.75   \\ \hline

    Sinkhorn-A & 0.49 & 0.47 & 0.47 & 0.55 & 0.52  & 0.74 & 0.80 & \textbf{0.72} &0.75 & 0.75   \\ 

    Sinkhorn-D & 0.47 & 0.44 & 0.45 & 0.59 & 0.52  &0.77 &0.84  &0.76 &\textbf{0.65} &0.76   \\ 

    Sinkhorn-U & 0.48 & 0.44 & 0.47 & 0.51 & 0.49  &0.77 &0.89  &0.77 &0.83 &0.82   \\ 

    Sinkhorn-E & 0.49 & 0.47 & 0.48 & 0.55 & 0.52  & \textbf{0.72} &\textbf{0.78}  &\textbf{0.72} &0.69 &\textbf{0.73}   \\ \hline
    
\end{tabular}}
\caption{Fine-grained SA task: Macro F1 and  Semantic Distance.} \label{tab:SA-1}
\end{center}
\end{table*}

\section{Experiments} \label{sec:expe_setup}
\paragraph{Baselines.} We setup the following baselines for a thorough comparison between Sinkhorn\footnote[6]{Here, we use KNOT and Sinkhorn interchangeably.} and entropy-based losses. Let [\textit{Method}] be the placeholder for Sinkhorn and Entropy. [\textit{Method}]-A denotes unweighted distillation of local models (\Cref{UnifDist}), i.e., $W_{B_k}(h_{k}(x))=1$ (in \Cref{eq:loss_final}). In [\textit{Method}]-D, $W_{B_k}(h_{k}({x}))$ is proportional to size of local datasets. [\textit{Method}]-U defines sample-specific confidence (weights) as the distance of output from the uniform distribution over labels. For each sample, [\textit{Method}]-E computes weight of $k\textsuperscript{th}$ local model as distance of its prediction from probability bias $B_k$, i.e., L\textsuperscript{2} confidence.

\begin{figure}
\begin{center}
\includegraphics[width=0.3\textwidth]{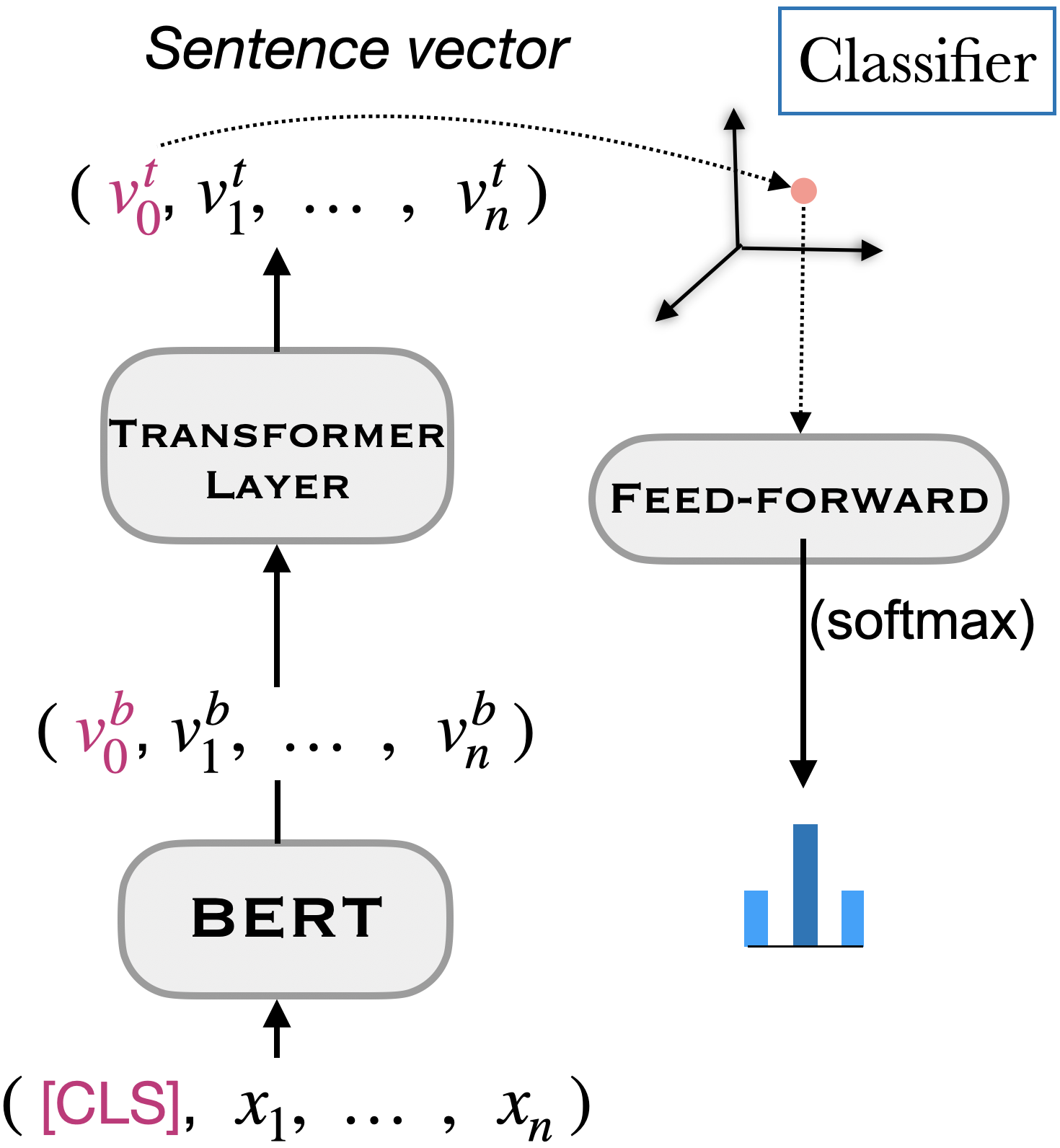}
\end{center}
\caption{Model architecture.} \label{fig:arch}
\end{figure}

\paragraph{Tasks.} We set up the three natural language understanding tasks that possess inter-class semantics: 1) fine-grained sentiment analysis (SA), 2) emotion recognition in conversation (ERC)~\cite{ghosal2022cicero,poria2020beneath}, and 3) natural language inference (NLI). NLI is the task of determining the inference relation between two texts. The relation can be entailment, contradiction, or neutral \cite{maccartney-manning-2008-modeling}. For a given transcript of a conversation, the ERC task aims to identify the emotion of each utterance from the set of pre-defined emotions \cite{poria2019emotion,hazarika2021conversational}. For our experiments, we choose the five most common emotions that are sadness, anger, surprise, happiness, and no emotion. 

\paragraph{Datasets.} For the SA task, we use four large-scale datasets: 1) Toys: toys and games; 2) Cloths: clothing and shoes; 3) Cell: cell phones and accessories; 4) Food: Amazon’s fine food reviews, specifically curated for the five-class sentiment classification. For transfer set (\Cref{sec:method}), we use grocery and gourmet food (104,817 samples) and discard the provided labels \cite{he2016ups}. Each dataset consists of reviews rated on a scale of 1 (strongly negative) to 5 (strongly positive). Similarly, for ERC, we collect three widely used datasets: DyDa: DailyDialog \cite{li2017dailydialog}, IEMOCAP: Interactive emotional dyadic motion capture database \cite{busso2008iemocap}, and MELD: Multimodal EmotionLines Dataset \cite{poria2018meld}. To demonstrate our methodology, we partition the DyDa dataset into four equal chunks. DyDa$_1$, DyDa$_2$, are used as local, DyDa$_3$ is used as global dataset. Dropping the labels from DyDa$_4$, we use it as a transfer set. For NLI task, we use SNLI \cite{bowman2015large} as global dataset and MNLI \cite{williams2017broad} as local dataset. We split the latter across its 5 genres, which are, fiction (Fic), government (Gov), telephone (Tele), travel (Trv), and Slate. This split assists in simulating distinct user (non-IID samples) setup. We use ANLI dataset \cite{nie2019adversarial} as a transfer set.

\begin{table*}[ht]
\begin{center}
\resizebox{0.95\textwidth}{!}{
\begin{tabular}{c|cccccc|cccccc}
\hline
\multirow{3}{*}{Algorithm} & \multicolumn{6}{c|}{F1 Score}  & \multicolumn{6}{c}{ Semantic Distance}    \\ 
    & \multicolumn{4}{c}{-------Local-------} & Global & \multirow{2}{*}{ALL} & \multicolumn{4}{c}{-------Local-------} & Global & \multirow{2}{*}{ALL} \\ 
    & MELD   & IEMOCAP   & DyDa$_0$  & DyDa$_1$ & DyDa$_2$  &   & MELD   & IEMOCAP   & DyDa$_0$  & DyDa$_1$ & DyDa$_2$   &   \\ 
Entropy-A  &0.28  &0.21  &0.34  &0.33  &0.36  &0.31  &0.67  &0.65  &0.60  &0.61  &0.63  &0.64   \\ 

Entropy-D  &0.30  &0.21  &0.39  &0.35  &0.38  &0.34  &0.68  &0.65  &0.57  &0.59  &0.61  &0.62   \\ 

Entropy-U  &0.30  &0.24  &0.42  &0.31  &0.37  &0.33  &0.68  &0.65  &0.60  &0.62  &0.64  &0.64   \\ 

Entropy-E  &0.34  &0.36  &0.42  &0.40  &0.44  &0.39  &0.69  &\textbf{0.62}  &0.57  &0.59  &0.62  &0.62   \\ \hline

Sinkhorn-A  &0.30  &0.26  &0.45  &0.37  &0.39  &0.34  &0.67  &0.67  &0.59   &0.62  &0.62   &0.64   \\ 

Sinkhorn-D  &0.35  &0.31  &0.45  &0.37  &0.44  &0.39  &0.67  &0.63  &0.54   &0.59  &0.61   &0.61   \\ 

Sinkhorn-U  &0.30  &0.23  &0.39  &0.34  &0.39  &0.34  &0.68  &0.68  &0.62   &0.64  &0.64   &0.65   \\ 

Sinkhorn-E  &0.38  &0.33  &0.46  &0.43  &0.43  &0.41  &\textbf{0.64}  &\textbf{0.62}  &\textbf{0.53}   &\textbf{0.56}  &\textbf{0.60}   &\textbf{0.59 }  \\ \hline

\end{tabular}}
\caption{ERC task: Macro F1 and Semantic Distance.}
\label{tab:SA-2}
\end{center}
\end{table*}

\begin{table*}
\begin{center}
\resizebox{0.95\textwidth}{!}{
\begin{tabular}{c|ccccccc|ccccccc}
\hline
\multirow{3}{*}{Algorithm} & \multicolumn{7}{c|}{Accuracy}  & \multicolumn{7}{c}{ Semantic Distance}    \\ 
    & \multicolumn{5}{c}{-------------Local-------------} & Global & \multirow{2}{*}{ALL} & \multicolumn{5}{c}{-------------Local-------------} & Global & \multirow{2}{*}{ALL} \\ 
    & Fic   & Gov   & Slate   & Tele & Trv & SNLI   &    & Fic   & Gov   & Slate   & Tele & Trv & SNLI   &       \\ 
Entropy-A  &0.60 &0.62 &0.60 &0.60 &0.62 &0.78 &0.65 & 0.56 &0.54 &0.56 &0.56 &0.55 &0.40 & 0.53   \\ 

Entropy-D  &0.58  &0.59  &0.57  &0.57  &0.58  &0.85  &0.64  & 0.58 &0.56  &0.58  &0.58  &0.57  &0.30  & 0.53   \\

Entropy-U  &0.60 &0.62 &0.60 &0.60 &0.61 &0.76 &0.65 & 0.55 &0.54 &0.56 &0.56 &0.55 &0.40 & 0.53   \\

Entropy-E  &0.60 &0.62 &0.60 &0.60 &0.61 &0.76 &0.65 & 0.55 &0.54 &0.56 &0.55 &0.54 &0.39 & 0.52   \\ \hline

Sinkhorn-A &0.60 &0.63 &0.60 &0.61 &0.62  &0.73 &0.64 &0.54 &0.53 &0.55 &0.55 &0.53 &0.42 &0.52   \\

Sinkhorn-D  &0.55  &0.57 &0.54 &0.55 &0.56 &0.85 &0.63 & 0.58 &0.56 &0.58 &0.58 &0.57 &\textbf{0.23} & 0.52   \\

Sinkhorn-U  &0.60 &0.62 &0.60 &0.60 &0.61 &0.77 &{0.65} &\textbf{0.53} &\textbf{0.52} &{0.55} &\textbf{0.54} &\textbf{0.52} &{0.37} &{0.51}   \\

Sinkhorn-E &0.60 &0.62 &0.60 &0.60 &0.61 &0.77 &0.65 &\textbf{0.53} &\textbf{0.52} &\textbf{0.54} &\textbf{0.54} &\textbf{0.52} &{0.37} &\textbf{0.50}   \\ \hline

\end{tabular}}
\caption{NLI task: Accuracy and  Semantic Distance.}
\label{tab:SA-3}
\end{center}
\end{table*}

\paragraph{Architecture.} We set up a compact transformer-based model used by both global and local models (\Cref{fig:arch}), although, the federation does not restrict both the local and model architectures to be the same. The input is fed to the pretrained BERT-based classifier \cite{devlin2018bert}. Thus, we obtain probabilities with support in the space of output labels, i.e., $\mathcal{Y}$. We keep all the parameters trainable, hence, BERT will learn its embeddings specific to the classification task. For the NLI task, we append premise and hypothesis at input separated by special token [\texttt{SEP}] token, followed by a standard classification setup.

\Cref{tab:SA-1}, \Cref{tab:SA-2}, and \Cref{tab:SA-3} show performance, i.e., Macro-F1 (or Accuracy) score and Semantic Distance of global models predictions from ground truth. Evaluations are done on fine-tuned (after distillation) global model with respect to the test sets of both local and global datasets. The testing over local datasets will help us analyze how well the domain generic global model performs over the individual local datasets and the testing over the global dataset is to make sure there is no catastrophic forgetting of the previous knowledge.

\paragraph{Training local models.} To compare Sinkhorn-based distillation with baselines, first, we pretrain local models. Since cross-entropy (CE) loss is less computationally expensive as compared to OT, we use CE for local model training. For all the models, we tuned hyperparameters and chose the model that performs best on the validation dataset. The data statistics and performances of local models on individual tasks are shown in \Cref{tab:stats_F1}.

\paragraph{Training global model.} We make use of transfer set samples to obtain noisy labels from local models. For a text sample in the transfer set, \Cref{eq:loss_final} aims to fit a global model to the weighted sum of predictions of the local models. To retain the previous knowledge of a global model and prevent catastrophic forgetting, we adapt the learning without forgetting paradigm. We store predictions of the pretrained global model on the transfer set and treat it similarly to the set of noisy labels obtained from the local models and perform its weighted distillation along with the local models.

\paragraph{Label-space.} We define label semantic spaces for the three tasks. As shown in  \Cref{fig:sem_cord}, we assign sentiment labels a one-dimensional space. For the ERC task, we map each label to a two-dimensional valence-arousal space. Valence represents a person's positive or negative feelings, whereas arousal denotes the energy of an individual's affective state. As mentioned in \cite{ahn2010asymmetrical}, anger (-0.4, 0.8), happiness (0.9, 0.2), no emotion (0, 0), sadness (-0.9, -0.4), and surprise (0.4, 0.9). The cost (loss) incurred to transport a mass from a point $p$ to point $q$ is $C_{p,q} \coloneqq |p-q|$. For NLI task, we define coordinates with entailment (1, 0, 0), contradiction (0,0,1) and neutral (0.5, 1, 0.5). The cost $C_{p,q} \coloneqq ||p-q||_2$, where, cost of transport from entailment to contradiction is higher than it is to neutral. It is noteworthy that for this task, we perform a manual search to identify label coordinates.

For the SA task in \Cref{tab:SA-1}, we observe the global models trained from Sinkhorn distillation of local models (contribution C1) perform better than corresponding Entropy-based variants on combined datasets (ALL) as well as on local and global datasets. Sinkhorn-A, D, U, and E are more semantically accurate in their predictions as compared to Entropy-A, D, U, and E, respectively. Moreover, we also notice that Entropy-E and Sinkhorn-E are better than corresponding A, D, and U variants, thus proving the utility of our contribution C2. 

For the ERC task in \Cref{tab:SA-2}, we observe the SD score of Sinkhorn-E is, in general, better amongst the Entropy and Sinkhorn-based baselines. In \Cref{tab:SA-3} of the NLI task, we notice the uniform weighting scheme performs as well as L\textsuperscript{2} on local datasets, however, it lags behind Sinkhorn-E in overall performance. As we observed for the SA task in Entropy-D and Sinkhorn-D settings, since the SNLI (global) dataset is bigger, the distillation forces the global model to perform better on the global dataset. It is observed to come at the cost of degraded performance on the other (local) datasets.

Comparing \Cref{tab:SA-1}, \Cref{tab:SA-2}, and \Cref{tab:SA-3} all together, for the three tasks with intrinsic similarity in the label space, we observe Sinkhorn-based loss transfer more semantic knowledge than entropy-based losses in the secure federated distillation setup. Moreover, we observe that L\textsuperscript{2} distance ([\textit{Method}]-E) gives better SD scores amongst the loss groups based on Entropy and Sinkhorn. Besides this, as compared to other baselines, empirical observations suggest that Sinkhorn-E (our combined contribution C1 and C2) works well for large-scale SA datasets, hence potentially scalable.

When we compare SA and ERC tasks with respect to standard metric scores, our method Sinkhorn-E is amongst the better performing models with the best accuracy and Macro F1 scores in 4 out of 5 tasks in SA and 4 out of 6 tasks in ERC. The model performs on par with baselines on the NLI task. Also, we find the Sinkhorn-based weighted distillation (Sinkhorn-E) shows a 2\% improvement on SA and ERC tasks while a 1\% average improvement on the NLI task when it is evaluated on the SD metric.

\section{Conclusion}
This work proposed KNOT, i.e., a novel optimal transport (OT)-based natural language semantic knowledge distillation. For the tasks with intrinsic label similarities, the OT distance between the predicted probability of the central (global) model and user-specific (local) models is minimized. To reduce the potential hazard of bias transfer from local model distillation, we introduced a weighting scheme based on the L\textsuperscript{2} distance between the local model's prediction and probability bias. Our experiments on three language understanding tasks —fine-grained sentiment analysis, emotion recognition in conversation, and natural language inference—show consistent semantic distance improvements while performing as good as the entropy-based baselines on the accuracy and F1 metrics.

\section*{Acknowledgement}
This work is supported by the A*STAR under its RIE 2020 AME programmatic grant RGAST2003 and project T2MOE2008 awarded by Singapore’s MoE under its Tier-2 grant scheme.

\bibliography{main}
\bibliographystyle{main}

\clearpage

\appendix

\onecolumn
\include{AppendixA}

\end{document}

%% file: math_commands.tex
\usepackage{amsmath,amsfonts,bm}

\def\eqref#1{equation~\ref{#1}}

\def\1{\bm{1}}

\DeclareMathAlphabet{\mathsfit}{\encodingdefault}{\sfdefault}{m}{sl}
\SetMathAlphabet{\mathsfit}{bold}{\encodingdefault}{\sfdefault}{bx}{n}

\newcommand{\KL}{D_{\mathrm{KL}}}

\DeclareMathOperator*{\argmin}{arg\,min}

%% file: AppendixA.tex
\label{append}

\section{What have we kept for the Appendix?}
We include proofs of \Cref{th:risk} and an analysis of the L\textsuperscript{2}-based confidence score. There are a few experiments that we consider to be important and may help compare the OT loss-based learning with Kullback–Leibler (KL). These results build a firm base to choose Sinkhorn-based (OT) losses on the task of federated distillation of sentiments. For the experiments, we work on the global model $\mathcal{M}_g$ that has acquired knowledge from local models in the learning without forgetting the paradigm.  

\begin{itemize}
    \item In \cref{Confidence_L2}, we provide an analysis of the L\textsuperscript{2}-based confidence metric.
    \item In \cref{DB}, we convey the intuition behind using a Sinkhorn distance over an entropy-based divergence. Moving further in \cref{DB1}, we show the importance of natural metrics in the label space by replacing the one-dimensional support with one-hot. Furthermore, in \cref{DB2}, we show how the model's clusters of sentence embeddings change when we move from a Sinkhorn distance-based loss to the KL divergence-based loss.
    \item In \cref{app_bias}, we discuss the potential risk of gender and racial bias transfer from the local models to the global models. Although we incorporate bias induced from non-IID data training of the local models, we do not tackle the transfer of other biases that can arise from the data as well as from the training process.
    \item In \cref{grad_comp}, we provide the algorithm to compute gradients and the computation complexity of Sinkhron loss.
    \item In \cref{risk_bound}, we provide the proof of \Cref{th:risk} on the empirical risk bound with the OT metric as unregularized 1-Wasserstein distance.
    \item In \cref{Soc_Impct}, we discuss the broader social impact of our work. We discuss how the method can be adopted for cyberbullying detection and the limitations coming from local models.
    \item In \cref{reprod}, we elaborate on the experimental settings and license of the datasets used in this paper.
\end{itemize}

\section{Confidence score (L\textsuperscript{2})} \label{Confidence_L2}

As shown in \Cref{fig:sink_conf} for a three-class classification, the equidistant distributions lie on an arc with a center at $B$. Points with high confidence scores lie on distant arcs. As radius of the arc increases, majority of its portion lies towards the high value of $p_l$, i.e., the $l$ with which the model is biased against since $b_l$ = min $\{b_1,\ldots,b_{|\mathcal{Y}|}\}$ ($p_3$ in the figure). Moreover, the maximum confidence score is achieved at the vertex $p_l=1$.

\begin{proposition}
From a given point $B$ in a k-simplex, point with the highest confidence lies on one of its vertices.
\end{proposition}

\begin{proof}
First, we analyze the case of a 2-simplex defined in a three-dimensional Euclidean space. Let $f_P =\sum_{i=1}^{3} (p_i - b_i)^2$, the quadratic program can be formulated as $\text{max}\{f_P  :\;\sum_{i=1}^{3}p_i=1,\; p_i\geq0\}$. The convex hull of vertices lying on the axes forms a closed and bounded feasible region. Thus, from the extreme value theorem, there exists absolute maximum and minimum. $f$ attains its minimum at $p=b$, which is also the critical point of $f_P$. Now, we need to find its value on the boundary points contained in the set of 1-simplices (line segments) $\{p_i+p_j=1, p_k=0: (i,j,k) \in {1,2,3}, i\neq j\neq k \}$. For the 1-simplex $p_1+p_2=1, p_3=0$, the values of $f_P$ at its endpoints that are $(1-b_1)^2 + b_2^2 + b_3^2$ and $b_1^2 + (1-b_2)^2 + b_3^2$, one of which is maxima of $f$ attained over the 1-simplex \footnote{Ignoring the critical point which gives the minima and perpendicular drawn from $b$ to the line segment.}. Similarly for the other line segments, the complete set of boundary values of $f_P$ is $k-2b_1$, $k-2b_2$, and $k-2b_3$ where $k = b_1^2 + b_2^2 + b_3^2+1$, occurring at $P=(1,0,0), (0,1,0) \text{ and } (0,0,1)$, respectively. Thus, the maximum of $f_P$ will lie on $i$\textsubscript{th}-axis such that $b_i = \text{min }(b_1, b_2, b_3)$. This proof can be generalized for a probability simplex in higher dimensions. As shown above, each iteration of a lower dimensional simplex will return vertices as the point of maxima in the end. 
\end{proof}

\begin{figure}
\centering
\includegraphics[width=0.3\textwidth]{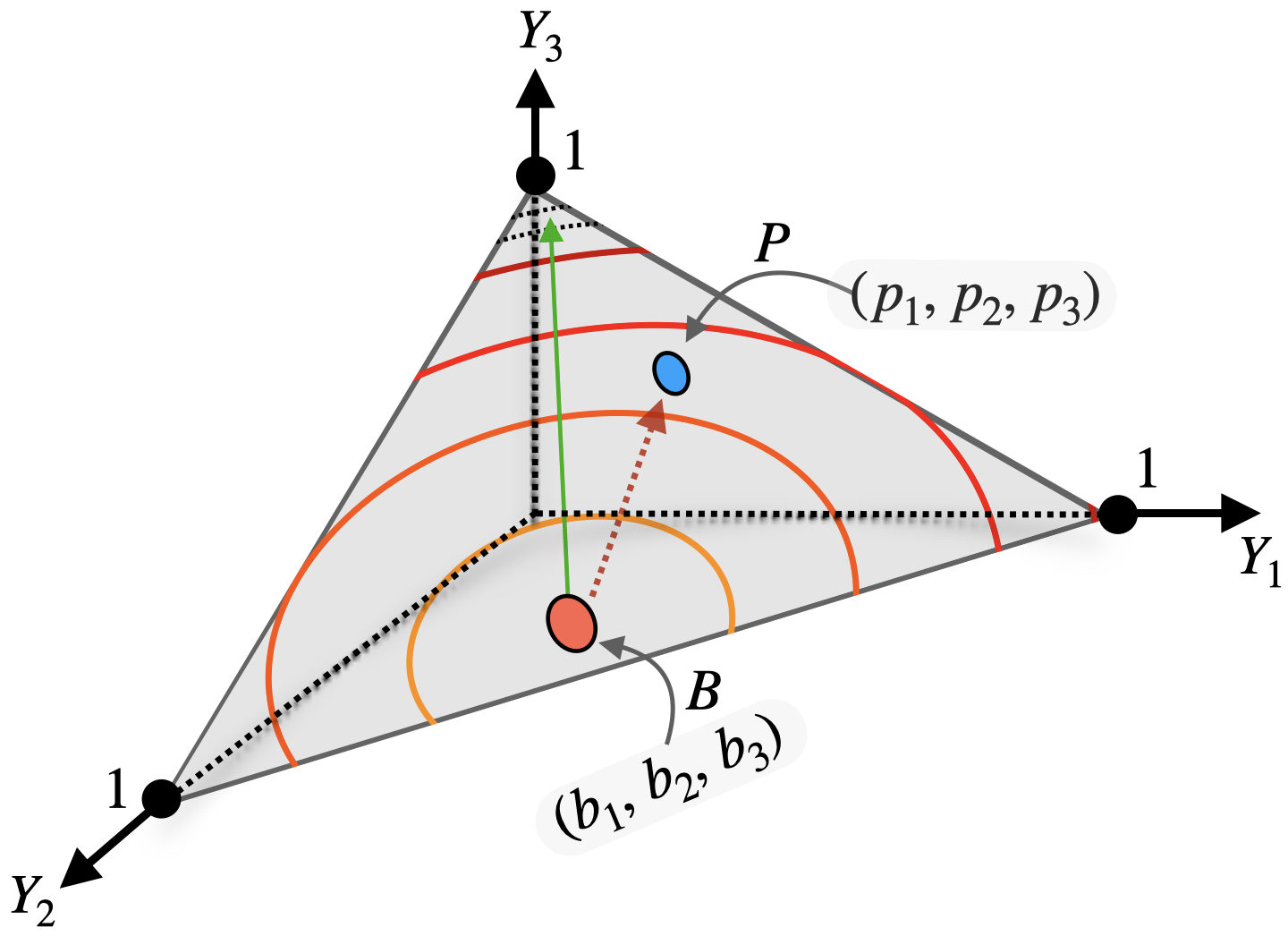}
\caption{\footnotesize{An illustration---a three-class classifier with bias $B$ and outputs a distribution $P~=~(p_1, p_2, p_3)$ for certain input. The green arrow denotes the direction of increased confidence score with equiconfident arcs.}} \label{fig:sink_conf}
\end{figure}

\section{Decision boundaries via sentence representations} \label{DB}

One of the main advantages of using Optimal Transport-based (OT) metrics between two probability distributions, such as Sinkhorn distance, is the ability to define the relationship in metric space. This is not feasible in entropy-based divergences. The relationship further appears in the loss function that accounts for the error computations of an intelligent system in the task of classification (or regression). With advancements in computations of Sinkhorn distances, as in \cite{feydy2019interpolating}, gradient computations through such loss functions have become more feasible as shown in \Cref{algorithm}. The inter-label relationships are apparent in tasks such as fine-grained sentiment classification, fake news detection, and hate speech. In this work, we consider the relationship between two labels $p$ and $q$ as a taxi-cab distance in the one-dimensional metric space of sentiment labels $\mathcal{Y}$. This relation nuance should appear in the Sinkhorn distance, we call it the cost of transportation from a point $p$ to another point $q$ in the set $\mathcal{Y}$.

When we set a learning algorithm to minimize the loss function, the goal is to find model parameters that provide the least empirical risk in the space of predefined hypotheses. From the distance-based cost (loss), the risk is expected to be minimum when the predicted labels are mapped "near" to the ground truth label. The term "near" refers to the lower optimal transport cost of the probability mass spread over a certain region to another region. In our problem, both the regions are the same, i.e., locations from 1 to 5 in the metric space. A ground truth probability mass (almost everything) at 1 would prefer an intelligent system to predict a probability mass near 1 so that it will require a lesser taxi-cab cost of transportation. It is noteworthy, that such relationships, even though apparent, are infeasible to appear in cost functions that inherit properties solely from the information theory.

\begin{figure}[ht]
    \centering
    \includegraphics[width=0.3\textwidth]{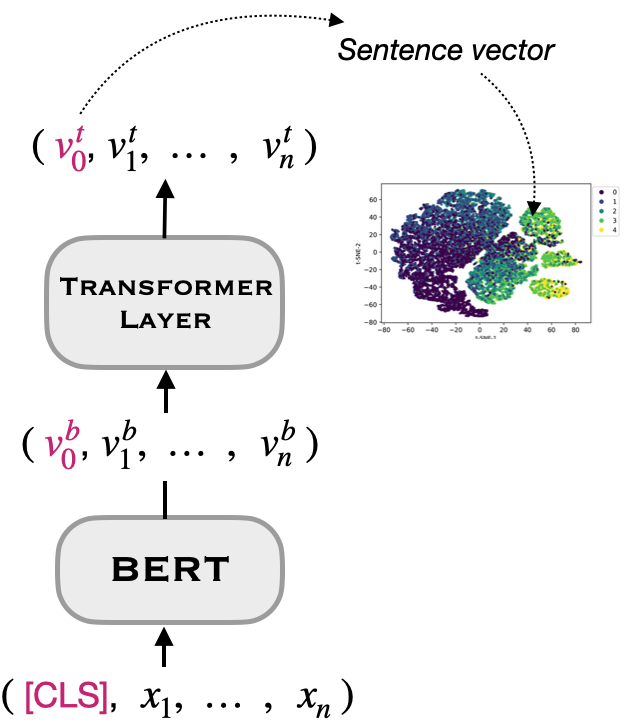}
    \caption{After the model parameter fitting is complete, we map the 128-dimensional [CLS] vector at the output of the Transformer layer to 2-dimensions using t-SNE.}
    \label{fig:arch_tsne} 
\end{figure}

\paragraph{t-SNE of sentence embeddings}
Next, we explain how we analyze the sentence embeddings in $\mathcal{M}_g$ obtained from \textit{Sink-E*}. A sentence refers to an Amazon food/product review. BERT's input sentence is lowercase WordPiece tokenized. We prepend the list of tokens with [CLS] token to represent the sentence which is later used for the classification task. First, each token is mapped to a static context-independent embedding. Then the vector list is passed through a sequence of multi-head self-attention operations that contextualizes each token. It is important to note that contextualization can be task-specific. We randomly sample 5000 review-label pairs for each sentiment class. For each textual review, as shown in \cref{fig:arch_tsne}, we use a 128-dimensional vector at the output of the transformer layer corresponding to the [CLS] token. This corresponds to the list of reviews represented in 128-dimensional vector space. To visualize the learned sentence representations, we map the vectors from 128-dimensional space to 2-dimensions using t-distributed Stochastic Neighbor Embedding (t-SNE)\footnote{We used the implementation from \href{https://scikit-learn.org/stable/modules/generated/sklearn.manifold.TSNE.html}{scikit-learn}.}.

\subsection{Semantic support to One-hot support} \label{DB1}
\begin{figure*}[ht]
    \centering
    \includegraphics[trim=0 550 10 10,clip,width=\textwidth]{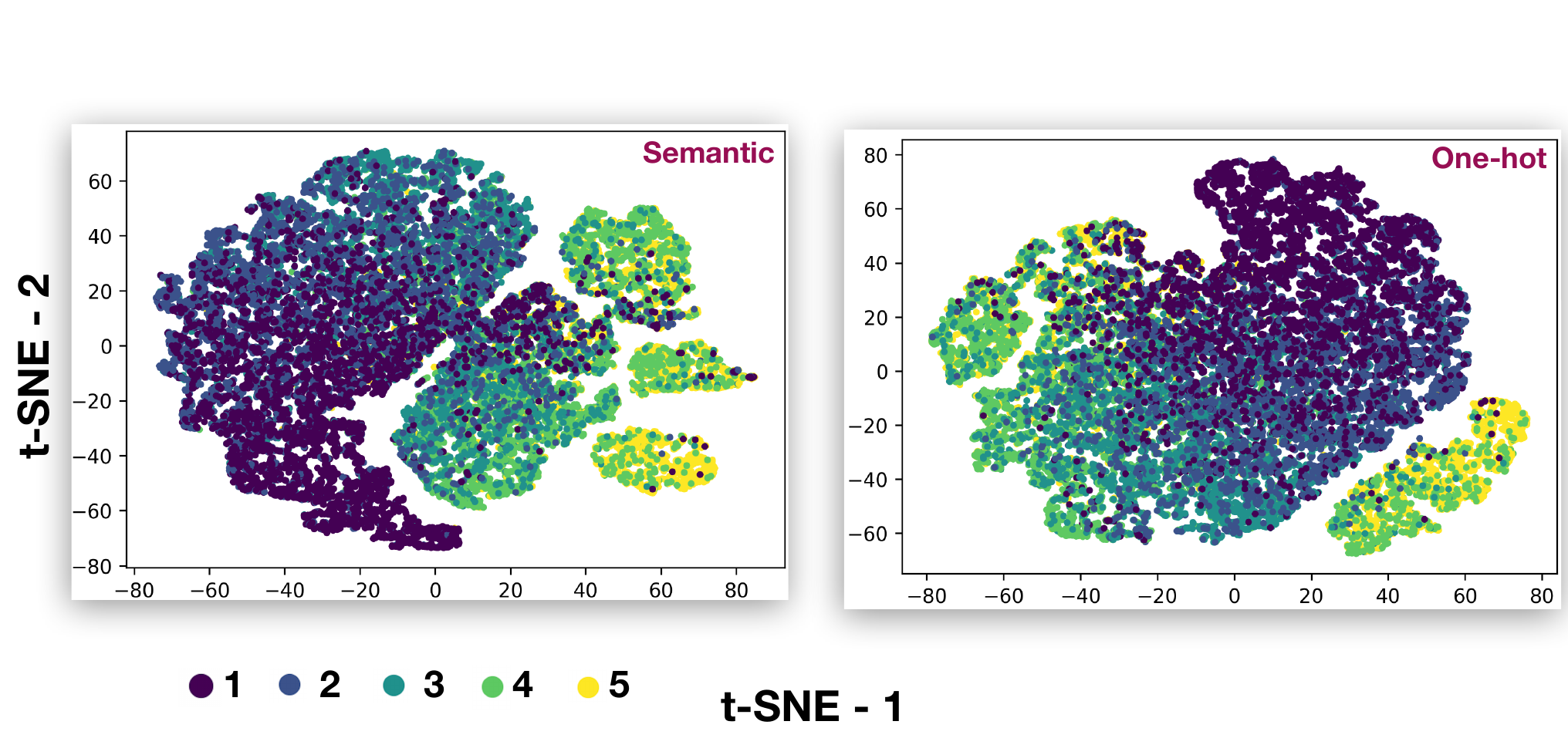}
    \caption{The one-hot encoding doesn't have clear distinct boundaries. The semantic structure is lost.}
    \label{fig:tsne2} 
\end{figure*}

One way to understand the importance of properly defining label space relationship is by defining metric space where each label acquires its own axis, thus losing the semantic information. For a five-class classification problem, we will have five axes and thus the support is a set of five distinct one-hot vectors each of size five. This way any misclassification, i.e., predicting a mass different from the ground truth label, will result in the same cost irrespective of whether positive sentiment is classified as strongly positive or strongly negative. This is due to the taxi-cab distance. Its value is computed by just summing up the absolute individual coordinate distances, which are just the predicted probabilities except for the coordinate corresponding to the ground truth.

We train the global model $\mathcal{M}_g$ with the Sinkhorn distance-based loss (\cref{eq:loss_final}) where the cost is defined as taxi-cab distance on the one-dimensional support and five-dimensional one-hot support as elaborated previously. The \cref{fig:tsne2} depicts the respective t-SNE scatter plots. In the plot with one-dimensional semantic support, we observe sentence vectors, i.e., features used for the classification task, are mapped in clearer clusters as compared to the plot at right without semantic information. We observe the points related to label 5 (strongly positive) are much more localized as compared to the sentence mappings with one-hot, which is distributed around the space. This clearly dictates the benefit of a meaningful metric as compared to a space that is not informative. Next, we check a similar case that occurs in entropy-based loss functions.

\subsection{Sinkhorn distance $\to$ KL divergence} \label{DB2}
Similar to the cost associated with the one-hot support in Sinkhorn, the KL divergence has no feasible way to capture the intrinsic metric in the label space. The plots in \cref{fig:tsne8} show the different sentence embeddings (t-SNE) with the varying entropy-based regularisation term in the vanilla Sinkhorn distance. As $\varepsilon \to 0^+$, we should get a pure OT-based loss function (\cref{eq:plain}). However, to speed up the Sinkhorn and gradient computations, we chose $\varepsilon = 0.001$ with no (F1-score) performance trade-off. As shown in the \cref{fig:tsne8}, with $\varepsilon >= 1$, the sentence representations are distributed across space with patches of label-dominant clusters. However, we can not see clear decision boundaries between the labels. As we decrease the $\varepsilon$ value below 1, we observe clearer feature maps for each label. For $\varepsilon$ = 0.001, we can see clear sentence vector clusters corresponding to label 5. We can see the higher confusion rate is only between labels 5 and 4 which can be attributed to the less cost of transportation of the mass from label 5 to 4 as compared to 5 to other labels. A similar trend can be seen for lower $\varepsilon$ values that are 0.01 and the $\varepsilon$ used in this work 0.003 where clearer and localized clusters can be seen.

\begin{figure*}[]
    \centering
\includegraphics[trim=0 100 10 10,clip,width=\textwidth]{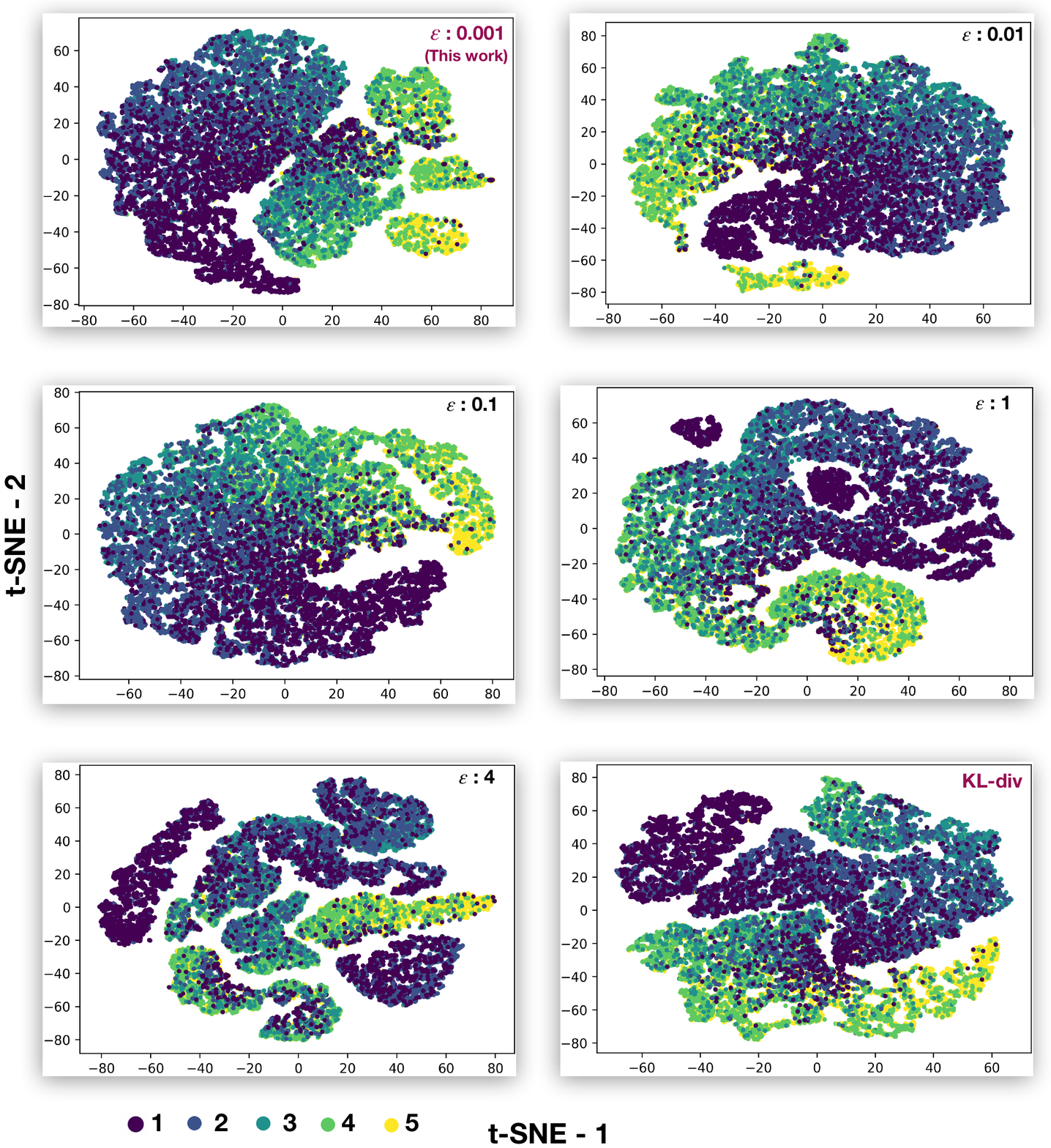}
    \caption{With a small entropic regularizer $\varepsilon$, it is visually striking that Sinkhorn seems to learn the latent structure boundaries better than KL divergence. For high $\varepsilon$, we see more clusters with mixed boundaries and not-so-clear demarcations.}
    \label{fig:tsne8}
\end{figure*}

\section{Model bias}\label{app_bias}
\subsection{Probability Skew}
To generate a random input, we uniformly sample 200 tokens from the vocab \footnote{We obtain the English vocabulary of size 30,522 from:\href{vocab}{\url{https://huggingface.co/google/bert_uncased_L-2_H-128_A-2/tree/main}}.
} with replacement and join them with white space. We obtain 100,000 such random texts. For a given text classifier model, the skew value for sentiment label 1 can be estimated by the fraction of times it is the prediction of when the model infers over the set of random texts (\cref{bias}).

\subsection{Gender and Racial biases}
Even though we considered the model probability skew as a reflection of bias induced from non-IID sampling, other biases such as gender and race can still be learned or acquired in the distillation process, For instance, take the following sentences:

\begin{quote}
    My \textbf{father} said that the food is just fine.  (review-1) $\to$ \textit{strong positive}\\
    My \textbf{mother} said that the food is just fine.  (review-2) $\to$ \textit{neutral}\\
\end{quote}

Review-1 and 2 differ in gender-specific words which are \textbf{father} and \textbf{mother}. Since it is a sentiment classification task, ideally, the intelligent system should not learn gender-specific cues from the text to generate its predictions. However, we observe a gender dependence in both the KL divergence and confident Sinkhorn-based predictions.

Similarly, we curate an example where the reviews differ only in a race-specific word.

\begin{quote}
    \textbf{White} guy said the phone is just fine.  (review-1) $\to$ \textit{neutral}\\
    \textbf{Latino} guy said the phone is just fine.  (review-2) $\to$ \textit{strong positive}\\
\end{quote}

The sentiment predictions made by the intelligent systems were different contrary to ideal behavior. The review with the word \textbf{White} shows a neutral sentiment while the review with the word \textbf{Latino} shows a strong positive sentiment. Hence, the systems took account of race-specific words while predicting the sentiment of a text. The behavior is observed both in the KL and Sinkhorn-based models with confident weights for sentiment classification.

\section{Gradients through loss}\label{grad_comp}
We demonstrate the computation of gradient of loss function (\cref{eq:loss_final}) with respect to the global model trainable parameters $\theta$. We can write the Lagrange dual of \cref{eq:plain} as
\begin{align} \label{eq:dual}
\begin{split}
    {\OT}_\varepsilon \defn \underset{(f,g) \in \mathcal{C}}{\max} \langle \mu, f \rangle & + \langle \nu, g \rangle - \varepsilon \langle \mu \otimes \nu, \exp \big(\frac{1}{\epsilon} (f \oplus g - \Cmat)\big) - 1 \rangle\\
    \mathcal{C} &= \{(f,g) \in \mathbb{R}^{n_s \times n_t}: f_i + g_j \leq {\Cmat}_{(i,j)}\}
\end{split}
\end{align}
where $f \oplus g$ is tensor sum $(y_s, y_t) \in \mathcal{Y}_s \times \mathcal{Y}_t \mapsto f(y_s) + g(y_t)$. The optimal dual (solution of \cref{eq:dual})  can retrieve us the optimal transport plan (solution of \cref{eq:loss_final}) with the relation $\pi = \exp(\frac{1}{\varepsilon})(f \oplus g - \Cmat)\cdot (\mu \otimes \nu)$. Recently, a few interesting properties of ${\OT}_\epsilon$ were explored \cite{peyre2019computational, feydy2019interpolating, luise2018differential} showing that optimal potentials $f$ and $g$ exist and are unique, and $\Delta {\OT}_\varepsilon (\mu, \nu) = (f,g)$.

\begin{algorithm}
\small
  \caption{Gradients of $\mathcal{L}(h_{\theta}(x), h_\mathcal{K}(x))$ with respect to $h_\theta(x)$
    \label{alg:packed-dna-hamming}}
\begin{algorithmic}[1]
\Require{Dual potentials \textit{\textbf{f \textsuperscript{1}, $\cdots$, f \textsuperscript{K}} $\in \mathbb{R}^{n_s}$} and \textit{\textbf{g \textsuperscript{1}, $\cdots$, g \textsuperscript{K}}} $\in \mathbb{R}^{n_t}$
}
\Statex
\For{$k \gets 1 \textrm{ to } K$}
\Let{\textit{\textbf{f\textsuperscript{k}}}}{\textbf{0}} \Comment{\textit{\textbf{f \textsuperscript{k}}}=$\{f^k_1, \ldots, f^k_{n_s}\}$}
    \Let{\textit{\textbf{g \textsuperscript{k}}}}{\textbf{0}} \Comment{\textit{\textbf{g\textsuperscript{k}}} = $\{g^k_1, \ldots, g^k_{n_t}\}$}
    \While{( \textit{\textbf{f \textsuperscript{k}, g \textsuperscript{k}}} not converged )} \tikzmark{top}
     \Let{$f_i^k$}{$\varepsilon \LSE_{m=1}^{n_s}(log(h^m_k(x)) + \frac{1}{\varepsilon} g_m - \frac{1}{\varepsilon} C(\mathcal{Y}_s^i, \mathcal{Y}_t^m))$}
     \Let{$g_j^k$}{$\varepsilon \LSE_{m=1}^{n_t}(log(h^{m}_\theta(x)) + \frac{1}{\varepsilon} f_m - \frac{1}{\varepsilon} C(\mathcal{Y}_s^m, \mathcal{Y}_t^j))$} \tikzmark{right} \tikzmark{bottom}\\
     ($\LSE$ is  log-sum-exp reduction, i.e, $\LSE_{m=1}^{M}(V_m) = \log \sum_{m=1}^{M} exp(V_m)$)
     \EndWhile
     \EndFor
     \\
$\frac{\partial({\OT}_\varepsilon(h_\theta(x), h_k(x)))}{\partial(h^{i}_\theta(x))}=f_i^k \;\;\;\;\;\forall i\in [n_s],\; k \in \mathcal{K}$ \Comment{as dual potentials are gradients of ${\OT}_\varepsilon$}.\\ 
$\frac{\partial({\mathcal{L}}_\varepsilon(h_\theta(x), h_\mathcal{K}(x)))}{\partial(h^{i}_\theta(x))}=\sum_{k \in \mathcal{K}} W_{B_k}(h_k(x)) f_i^k/ \sum_{k \in \mathcal{K}} W_{B_k}(h_k(x))$.

\end{algorithmic} \label{algorithm}
  \AddNote{top}{bottom}{right}{Sinkhorn loop.}
\end{algorithm}

Using these properties, we calculate gradients of the confident Sinkhorn cost in $\cref{eq:loss_final}$. \Cref{algorithm} obtains the gradients of the loss function with respect to $h_\theta(x)$ which can be backpropagated to tune model parameters. A crucial computation is to solve the coupling equation in steps 5 and 6. This is done via Sinkhorn iterations which have a linear convergence rate \cite{peyre2019computational}.

\section{Statistical Risk Bounds}\label{risk_bound}
Without the loss of generality, we will prove the risk bounds for two local models in learning without forgetting the paradigm. For a sample $x$, let the output of local models be $y_1=h_1(x)$ and $y_2=h_2(x)$ and the global model with trainable parameters be $y_{\theta}=h_\theta(x)$. To prove \Cref{th:risk}, we consider the set of IID training samples $S=\{(x^{(1)}, y_1^{(1)}, y_2^{(1)}), \ldots, (x^{(N)}, y_1^{(N)}, y_2^{(N)})\}$.

\begin{lemma} (from \cite{frogner2015learning})\label{a1} Let $h_{\hat{\theta}}, h_{\theta^*} \in \mathcal{H}_g$ be the minimizer of empirical risk $\hat{R}_S$ and expected risk $R$, respectively. Then

\begin{equation}
    R(h_{\hat{\theta}}) \leq R(h_{\theta^*}) + 2 \underset{h_\theta \in \mathcal{H}_g}{\sup} |R(h_\theta) - \hat{R}_S(h_\theta)| 
\end{equation}
\end{lemma}

To bound the risk for $h_{\hat{\theta}}$, we need to prove uniform concentration bounds for the distillation loss. We denote the space of loss functions induced by hypothesis space $\mathcal{H}_{g}$ as

\begin{equation}
    L = \Bigg\{\ell_\theta:(x,y_1, y_2) \mapsto \frac{w_1(y_1) D(y_\theta,y_1) + w_2(y_2)D(y_\theta,y_2)}{w_1(y_1) + w_2(y_2)} \Bigg\}
\end{equation}

\begin{lemma} \label{lem5} (\cite{frogner2015learning})
Let the transport cost matrix be $C$ and the constant $C_M = \underset{(i,j)}{\max}\;C_{(i,j)}$, then $0 \leq D(\cdot,\cdot) \leq C_M$, where $D(\cdot,\cdot)$ is \text{1}-Wasserstein distance.
\end{lemma} 

\begin{definition}
(The Rademacher Complexity \cite{bartlett2002rademacher}). Let $G$ be a family of mapping from $\mathcal{Z}$ to $\mathbb{R}$, and $S=(z_1,\ldots,z_N)$ a fixed sample from $\mathcal{Z}$. The empirical Rademacher complexity of $G$ with respect to $S$ is defined as:

\begin{equation}
    \hat{\mathfrak{R}}_S(G)= {\rm I\!E}_{\sigma}\Bigg [\underset{g \in G}{\sup}\;\frac{1}{N} \sum_{i=1}^{n} {\sigma}_i g(z_i) \Bigg]
\end{equation}

where $\sigma = (\sigma_1,\ldots,\sigma_N)$, with $\sigma_i$'s independent uniform random variables taking values in $\{+1,-1\}$. $\sigma_i$'s are called the Rademacher random variables. The Rademacher complexity is defined by taking expectation with respect to the samples $S$.

\begin{equation}
    \mathfrak{R}_N(G) = {\rm I\!E}_S\big[ \hat{\mathfrak{R}}_S(G)\big]
\end{equation}
\end{definition}

\begin{theorem} \label{a3}
 For any $\delta > 0$, with probability at least 1-$\delta$, the following holds for all $l_{\theta} \in L$:
 \begin{equation}
   {\rm I\!E}[\ell_\theta] - \hat{{\rm I\!E}}[\ell_\theta] \leq 2\mathfrak{R}_N(L) + \sqrt{\frac{C_M^2log(1/\delta)}{2N}}.
 \end{equation}
\end{theorem}

\begin{proof}
By definition ${\rm I\!E}[\ell_\theta] = R(h_\theta)$ and $\hat{{\rm I\!E}}[\ell_\theta] = \hat{R}(h_\theta)$. Let,

\begin{equation*}
    \Phi(S)=\underset{\ell \in L}{\sup} \; {\rm I\!E}[\ell] - \hat{{\rm I\!E}}_S[\ell].
\end{equation*}

Let $S$ and $S'$ differ only in sample $(\Bar{x}^{(i)}, \Bar{y}^{(i)}_1, \Bar{y}^{(i)}_2)$, by Lemma \ref{lem5}, it holds that:

\begin{multline}
    \Phi (S) - \Phi (S') \leq \underset{\ell \in L}{\sup}\; \hat{{\rm I\!E}}_{S'}-\hat{{\rm I\!E}}_{S} = \underset{h_\theta \in \mathcal{H}}{\sup} \frac{1}{N} \Bigg \{w_1(\Bar{y}^{(i)}_1)D(\Bar{y}_\theta^{(i)},\Bar{y}^{(i)}_1) + w_2(\Bar{y}^{(i)}_2)D(\Bar{y}_\theta^{(i)},\Bar{y}^{(i)}_2) \\
    - w_1(y^{(i)}_1)D(y_\theta,y_1^{(i)}) - w_2(y_2^{(i)})D(y_\theta,y_2^{(i)}) \Bigg \} \leq \frac{2C_M}{N}
\end{multline}
This inequality can be achieved by putting $D(\Bar{y}_\theta^{(i)},\Bar{y}^{(i)}_1)=D(\Bar{y}_\theta^{(i)},\Bar{y}^{(i)}_2)=C_M$ and $D(y_\theta^{(i)},y_1^{(i)})=D(y_\theta^{(i)},y_2^{(i)})=0$.

Similarly, $\Phi (S') - \Phi (S) \leq C_M/N$, thus $|\Phi (S') - \Phi (S)| \leq C_M/N$. Now, from the McDiarmid’s inequality \cite{mcdiarmid1998concentration} and its usage in \cite{frogner2015learning}, we can establish

\begin{equation}
    \Phi (S) \leq {\rm I\!E}[\Phi(S)] + \sqrt{\frac{KC_M^2log(1/\delta)}{2N}}.
\end{equation}

From the bound established in the proof of Theorem B.3 in \cite{frogner2015learning}, i.e., $ {\rm I\!E}_S[\Phi(S)] \leq 2\mathfrak{R}_N(L)$, we can conclude the proof.
\end{proof}

To complete the proof of Theorem \Cref{th:risk}, we have to treat $\mathfrak{R}_N(L)$ in terms of $\mathfrak{R}_N(\mathcal{H}_g)$.

Now, let $\iota: \mathbb{R}^{|\mathcal{Y}|} \times \mathbb{R}^{|\mathcal{Y}|} \mapsto \mathbb{R}$ defined by $\iota(y, y') = D(\mathcal{s}(y),\mathcal{s}(y)')$, where $\mathcal{s}$ is a softmax function defined over the vector of logits. From Proposition B.10 of \cite{frogner2015learning}, we know:

\begin{equation}
    |\iota(y, y') - \iota(\bar{y}, \bar{y}')| \leq 2C_M(||y-\bar{y}||_2 + ||y'-\bar{y}'||_2)
\end{equation}

Let $\iota_s: \mathbb{R}^{|\mathcal{Y}|} \times \mathbb{R}^{|\mathcal{Y}|} \times \mathbb{R}^{|\mathcal{Y}|} \mapsto \mathbb{R}$ defined by:

\begin{align}
    \iota_s(y,y_1,y_2) &= \frac{w_1(\mathfrak{s}(y_1)) D(\mathfrak{s}(y),\mathfrak{s}(y_1)) + w_2(\mathfrak{s}(y_2))D(\mathfrak{s}(y),\mathfrak{s}(y_2))}{w_1(\mathfrak{s}(y_1)) + w_2(\mathfrak{s}(y_2))}\\
    &= \bar{w}_1(\mathfrak{s}(y_1), \mathfrak{s}(y_2)) \;D(\mathfrak{s}(y),\mathfrak{s}(y_1)) + \bar{w}_2(\mathfrak{s}(y_1), \mathfrak{s}(y_2))\;D(\mathfrak{s}(y),\mathfrak{s}(y_2))
\end{align}

where $w_1(.), w_2(.)$ are confidence score of local model predictions $y_1, y_2$ on an input $x$. $\bar{w}_1(.), \bar{w}_2(.)$ are normalized scores. Note that the local model predictions, i.e., $y_1$ and $y_2$ are functions of $x$, where $x$ is sampled from the data domain distribution $f(x)$. Hence, we can view the loss function as
\begin{align}
    \iota_s(y,y_1,y_2) &= D(\mathfrak{s}(y),\mathfrak{s}(y_1)) + D(\mathfrak{s}(y),\mathfrak{s}(y_2))\\
    &= \iota(y,y_1) + \iota(y,y_2).
\end{align}

where $y$ is a function of $x_{new}$ sampled from a weighted distribution $\bar{w}_1(\mathfrak{s}(y_1), \mathfrak{s}(y_2)) f(x)$.

The \textit{Lipschitz} constant of $\iota_s(y,y_1,y_2)$ can thus be identified by:

\begin{align}
    |\iota_s(y,y_1,y_2)-\iota_s(\bar{y},\bar{y}_1,\bar{y}_2)| &= |\iota(y,y_1) + \iota(y,y_2) - \iota(\bar{y},\bar{y}_1) - \iota(\bar{y},\bar{y}_2)|\\
    &\leq |\iota(y,y_1) - \iota(\bar{y},\bar{y}_1)| + |\iota(y,y_2) + \iota(\bar{y},\bar{y}_2)|\\
    &\leq 2C_M(||y-\bar{y}||_2 + ||y_1-\bar{y}_1||_2 + ||y-\bar{y}||_2 + ||y_2-\bar{y}_2||_2)\\
    &\leq 4C_M(||y-\bar{y}||_2 + ||y_1-\bar{y}_1||_2 + ||y_2-\bar{y}_2||_2)\\
    &\leq 4C_M||(y,y_1,y_2)-(\bar{y},\bar{y}_1,\bar{y}_2)||_2
\end{align}

Thus, the \textit{Lipschitz} constant of plain Sinkhorn based distillation is $4C_M$.

\paragraph{Proof of Theorem \ref{th:risk}}
We define the space of loss function for $k$ local models:
\begin{equation*}
    L = \Bigg\{\iota_\theta:(x,\{y_k\}_{k \in \mathcal{K}}) \mapsto \sum_{k \in \mathcal{K}} w_k({\mathfrak{s}(h_\theta^o(x))})D(\mathfrak{s}(y_k))\Bigg\}
\end{equation*}

Following the notations in \cite{frogner2015learning}, we apply the following generalized Talagrand’s lemma \cite{ledoux2013probability}:
\begin{lemma}
Let $\mathcal{F}$ be a class of real functions, and $\mathcal{H} \subset \mathcal{F} = \mathcal{F}_1 \times \ldots \times \mathcal{F}_K$ be a $K$-valued function class. If $m : \mathbb{R}^K \mapsto \mathbb{R}$ is a \textit{$L_{\mathfrak{m}}$-Lipschitz} function and $\mathfrak{m}(0) = 0$, then $\mathfrak{R}_S(\mathfrak{m} \circ \mathcal{H}) \leq 2L_{\mathfrak{m}} \sum_{k=1}^{K} \hat{\mathfrak{R}}_S(\mathcal{F}_k)$.
\end{lemma}

Now, the Lemma can not be directly applied to the confident Sinkhorn loss as \textbf{0} is an invalid input. To get around the problem, we assume the global hypothesis space is of the form:

\begin{equation}
    \mathcal{H} = \{\mathfrak{s} \circ h : h^o \in \mathcal{H}^o\}
\end{equation}

Thus, we apply the lemma to the \textit{$4C_M$-Lipschitz} continuous function $\mathfrak{l}$ and the function space:

\begin{equation*}
    \underbrace{\mathcal{H}^o \times \ldots \times \mathcal{H}^o}_{|\mathcal{Y}| copies}\underbrace{\times \mathcal{I} \times \ldots \times \mathcal{I}}_{|\mathcal{Y}|\times|\mathcal{K}| copies}
\end{equation*}

with $\mathcal{I}$ a singleton function space of identity maps. It holds:
\begin{equation}\label{eq:rn}
    \mathfrak{R}_N(L) \leq 8C_M (|\mathcal{Y}|\hat{\mathfrak{R}}_N + |\mathcal{Y}|\times|\mathcal{K}|\hat{\mathfrak{R}}_N(\mathcal{I})) = 8|\mathcal{Y}|C_M\hat{\mathfrak{R}}_N(\mathcal{H}^o)
\end{equation}

As,
\begin{equation*}
    \hat{\mathfrak{R}}_N(\mathcal{I}) = {\rm I\!E}_{\sigma}\Big[\underset{g\in \mathcal{I}}{\sup}\; \frac{1}{N}\sum_{i=1}^{N}\sigma_i g(y_i) = 0\Big] =  {\rm I\!E}_{\sigma}\Big[\frac{1}{N}\sum_{i=1}^{N}\sigma_i y_i = 0\Big]
\end{equation*}

Thus, by combining \cref{eq:rn} in Theorem \ref{a3} and Lemma \ref{a1} proves the Theorem \ref{th:risk}.

Since the $\varepsilon$ is small in our experiments, we can quantify the difference between Sinkhorn distance and Wasserstein for a given \textit{Lipschitz} cost function.



\section{Societal Impact}\label{Soc_Impct}
Our work can be extended to different domains. Although in this paper we examined sentiment classifications, other areas, where labels are not available, i.e., zero-shot classification, would also be amenable to federated confident Sinkhorns. Within our approach, a potential downstream task could be to detect cyberbullying. An important area of application for distraught parents, school teachers, and teens. In this case, a sentiment that has a high probability of being classified as cyberbullying can be flagged to either moderators or guardians of a particular application.  

A weakness of this approach is that the training of such an application will be based on local models in other domains. Care would be needed in deciding which local models to use in the federation. This choice is highly dependent on the industry and the availability of data. 
Misuse of our approach could be that the federated training might distill some population-specific information to the global model which makes the central system vulnerable to attacks that might lead to a user-private data breach. As GPU implementation of OT metrics becomes commonplace, we envisage that our approach might help in other Natural Language Processing (NLP) tasks. Indeed, this would be potentially beneficial and open up new avenues for the NLP community. 

\section{Experimental reproducibility}\label{reprod}
All the experiments were performed on one Quadro RTX 8000 GPU with 48 GB memory. The model architectures were designed on Python (version 3.9.2) library PyTorch (version 1.8.1) under \textit{BSD-style license}. For Sinkhorn iterations and gradient calculations, we use GeomLoss library (version 0.2.4) (\url{https://github.com/jeanfeydy/geomloss}) under \textit{MIT licence}. For barycenter calculations, we use POT library (0.7.0) from (\url{https://pythonot.github.io/}) under \textit{MIT licence}.

We clip all the text to a maximum length of 200 tokens and pad the shorter sentences with <unk>. To speed up the experiments, we use pretrained BERT-Tiny from \url{https://github.com/google-research/bert}.

The batch size is chosen via grid search from the set $\{16,32,64,128,256,512,1024,2048\}$ and found 1024 to be optimal for performance and speed combination on the considered large datasets. We use Adam optimizer with learning rate chosen via grid search $\{10^{-4},10^{-3},10^{-2},10^{-1},1,10^{1},10^{2},10^{3},10^{4}\}$. All the experiments were run for 20 epochs. The regularization parameter $\varepsilon$ is chosen based on minimal loss obtained amongst the set of $\varepsilon$ values $\{10^{-4},10^{-3},10^{-2},10^{-1},1,2,4,8,16\}$.

For the Amazon review dataset, we were unable to find the license.